\documentclass[sigconf]{acmart}
\AtBeginDocument{%
  }
\usepackage{algorithm}
\usepackage{algpseudocode}
\usepackage{array}
\usepackage{multirow}
\newcolumntype{M}[1]{>{\centering\arraybackslash}m{#1}}
\newtheorem{theorem}{Theorem}
\newtheorem{proposition}{Proposition}
\newtheorem{definition}{Definition}
\newtheorem{lemma}{Lemma}
\newtheorem{remark}{Remark}
\newtheorem{assumption}{Assumption}

\newtheorem{corollary}{Corollary}

\begin{document}

\title{CasualSynth: Generating Structurally Sound Synthetic Data}

\author{Zehua Cheng}
\affiliation{%
  \institution{Department of Computer Science, University of Oxford}
  \country{}
}
\email{zehua.cheng@cs.ox.ac.uk}

\author{Wei Dai, Jiahao Sun}
\affiliation{%
  \institution{FLock.io}
  \country{United Kingdom}
}
\email{sun@flock.io}
\author{Thomas Lukasiewicz}
\affiliation{%
  \institution{Institute of Logic and Computation, TU Wien}
  \country{}
}
\affiliation{%
  \institution{Department of Computer Science, University of Oxford}
  \country{}
}
\email{thomas.lukasiewicz@tuwien.ac.at}

\renewcommand{\shortauthors}{Cheng et al.}

\begin{abstract}
  Large Language Models (LLMs) generate realistic synthetic data but offer no guarantee that their outputs respect the causal mechanisms governing the target domain. We introduce \textbf{CausalSynth}, a framework that decouples causal structure generation from semantic realization, yielding synthetic data that is both causally valid and linguistically rich. The framework operates in three phases. First, a Structural Causal Model (SCM)---a tuple of structural equations defined over a directed acyclic graph (DAG)---generates \emph{causal skeletons}, i.e., variable assignments that satisfy the Global Markov Property of the governing DAG, via ancestral sampling. Second, an LLM acts as a constrained \emph{realizer}, a conditional translator that maps each skeleton to a high-dimensional observation such as a clinical note or a transaction log. Third, an Iterative Consistency Verification module detects structural violations through deterministic extraction and feeds targeted corrections back to the LLM, forming a closed-loop refinement process. We identify the \emph{Semantic Backdoor} problem---the systematic tendency of LLMs to override imposed causal facts with pre-training priors---and prove that our iterative mechanism reduces the resulting selection bias relative to standard rejection sampling. On three causal benchmarks (ASIA, ALARM, and MIMIC-Struct), CausalSynth preserved conditional independencies with false-positive rates near the nominal $\alpha=0.05$ level and achieved realizability rates above 96\% with 70B-parameter LLM backbones. The framework additionally supports principled interventional and counterfactual generation through noise retention and graph mutilation.
\end{abstract}



\maketitle

\section{Introduction}

Large Language Models (LLMs) have transformed data generation, enabling the synthesis of text that is often indistinguishable from human-authored content~\cite{brown2020language,openai2023gpt4}. This capability promises to accelerate model training, to augment sparse datasets, and to preserve privacy by substituting real records with synthetic counterparts~\cite{nikolenko2021synthetic,jordon2022synthetic}. As synthetic data increasingly underpins downstream decision systems, the fidelity of these generations becomes critical. Yet a tension persists between linguistic fluency and structural reliability: LLMs operate as probabilistic correlation engines that prioritize plausibility over logical consistency~\cite{bender2021dangers,marcus2018deep}. The distinction between surface-level correlation and underlying causal mechanism is not merely academic, but a matter of safety and efficacy in high-stakes deployments~\cite{kiciman2023causal}.

In scientific inquiry, policy-making, and medical diagnostics, the utility of data depends on its representation of causal truth, not mere association. A dataset used to train a diagnostic system must reflect the causal pathway from pathology to symptom, not just their co-occurrence in electronic health records. If a generative model reverses this relationship or introduces spurious confounders rooted in training-data biases, downstream models are fundamentally flawed~\cite{pearl2018book}. The challenge is therefore not to generate data that \textit{looks} real, but to generate data that faithfully respects the governing causal laws of the domain. We refer to this property as \emph{structural validity}: the requirement that the joint distribution of the generated data factorize according to the causal graph $\mathcal{G}$ of the target domain. Structural validity is indispensable for causal-inference tasks such as predicting the effects of medical interventions or economic policies, which require reasoning about counterfactual scenarios outside the observed training distribution~\cite{peters2017elements}.

Existing approaches for controlling LLM generation are insufficient for enforcing such structural validity. Constrained decoding and schema enforcement guarantee valid output formats (e.g., well-formed JSON) but are agnostic to semantic correctness~\cite{willard2023efficient,geng2023grammar}. Fine-tuning approaches for tabular data generation~\cite{borisov2023language} improve statistical fidelity but inherit no explicit structural blueprint. Knowledge-graph-augmented generation grounds outputs in factual relations, yet standard knowledge graphs encode static facts rather than the functional mechanisms required to model cause and effect~\cite{pan2024unifying}. Practitioners are thus forced to choose between the rich narrative capabilities of LLMs and the rigorous output of formal simulators. No unified framework yet leverages the generative creativity of LLMs while adhering to a mathematically rigorous causal blueprint.

To address this gap, we propose \textbf{CausalSynth}, which formally decouples the generation of causal structure from the realization of surface-level features. The \emph{causal skeleton} of a data point---a vector of variable assignments together with the mechanistic relationships linking them---is determined by a Structural Causal Model (SCM), the standard representation of causal knowledge~\cite{pearl2009causality}. Ancestral sampling from the SCM~\cite{bishop2006pattern,koller2009probabilistic} ensures that the latent structure is mathematically sound and free from spurious correlations. The skeleton is then passed to an LLM acting as a constrained \emph{realizer} that translates the abstract assignments into a coherent narrative. An Iterative Consistency Verification mechanism detects and corrects semantic violations via a feedback-driven refinement loop, provably reducing the selection bias caused by the LLM's non-uniform realizability across the skeleton space.

Our main contributions are as follows:
\begin{itemize}
  \item We introduce \textbf{CausalSynth}, a framework that integrates SCMs with LLMs to generate high-fidelity synthetic data whose causal structure is guaranteed by construction rather than learned implicitly. We provide a formal proof that the generated skeletons satisfy the Global Markov Property of the governing DAG (Theorem~\ref{theorem:structural-fidelity-ancestral-sampling}).
  \item We identify the \textbf{Semantic Backdoor} problem---the systematic tendency of LLMs to override imposed causal constraints with pre-training priors---and propose an Iterative Consistency Verification mechanism to mitigate it. We prove that this mechanism reduces selection bias relative to standard rejection sampling (Theorem~\ref{theorem:bias-red-via-iterative-refine}).
  \item We show on three causal benchmarks (ASIA, ALARM, MIMIC-Struct) that CausalSynth preserves conditional independencies with false-positive rates near the nominal level ($\approx 0.05$) and achieves realizability rates above 96\% with large-scale LLM backbones.
  \item We further validate the framework's support for interventional and counterfactual generation, including comparisons against modern tabular baselines (TabDDPM, GReaT) and a sensitivity analysis of the NLI verifier threshold.
\end{itemize}

\section{Related Works}

\subsection{Synthetic Data Generation}
Synthetic data generation has evolved from classical statistics to deep generative models. Early work centred on simulation-based generators using predefined parametric models~\cite{rubin1993statistical}. The advent of Generative Adversarial Networks (GANs)~\cite{goodfellow2014generative} and Variational Autoencoders (VAEs)~\cite{kingma2014auto} marked a paradigm shift, enabling implicit density estimation in high-dimensional spaces. For tabular data, CTGAN and TVAE~\cite{xu2019modeling} adapted these architectures to mixed-type columns and to mode collapse in categorical variables. More recently, diffusion models~\cite{kotelnikov2023tabddpm} and transformer-based approaches have shown promise; GReaT~\cite{borisov2023language} fine-tunes LLMs to generate tabular rows as serialized text. All of these methods approximate the joint distribution $P(\mathbf{V})$ without an explicit structural model, and as Table~\ref{tab:main-results} shows, this distributional agnosticism produces large violations of conditional independence---a failure mode that our SCM-first architecture avoids by construction.

A growing body of work uses LLMs as data generators, ranging from zero-shot prompting for dataset creation~\cite{veselovsky2023generating} to fine-tuning strategies that adapt foundation models to specific distributions~\cite{borisov2023language}. Chain-of-Thought prompting~\cite{wei2022chain} improves the reasoning capabilities of LLMs, and we use it to increase the model's adherence to structural constraints during realization. Recent studies of LLMs' causal-reasoning ability find that they can retrieve causal knowledge embedded in pre-training data but struggle with novel or counter-intuitive scenarios~\cite{kiciman2023causal,jin2024cladder}. This finding motivates our architectural choice: rather than relying on the LLM's implicit causal reasoning, we externalize the causal computation to the SCM and treat the LLM strictly as a conditioned text generator, with an external verifier responsible for detecting and correcting structural violations.

\subsection{Causal Modeling and Structural Causal Models}
Structural Causal Models, formalized by Pearl~\cite{pearl2009causality,pearl2000models}, provide the mathematical foundation for reasoning about interventions ($do$-calculus) and counterfactuals. The practical deployment of SCMs depends on the availability of a causal graph, which may be specified by domain experts or discovered from data using algorithms such as PC~\cite{spirtes2000causation}, GES~\cite{chickering2002optimal}, and continuous-optimization methods like NOTEARS~\cite{zheng2018dags}. A fundamental identifiability limit constrains all observational approaches: the true causal graph is recoverable only up to a Markov Equivalence Class (MEC)~\cite{verma1990equivalence}. We acknowledge this limit explicitly by distinguishing between Oracle and Learned settings and by conditioning all structural guarantees on the assumed model $\mathcal{M}$.

A growing line of work develops deep generative models that respect a known SCM. These methods share a common motivation: they encode the DAG structure into the generator so that interventional and counterfactual queries can be answered without re-training. \textsc{CausalGAN}~\cite{kocaoglu2018causalgan} parameterizes the structural equations with adversarial networks. \textsc{VACA}~\cite{sanchezmartin2022vaca} encodes DAG-structured priors in a variational graph autoencoder for counterfactual queries. Causal Normalizing Flows~\cite{javaloy2023causalnf} retrofit invertibility onto $\mathbf{F}$ to enable closed-form interventions. \textsc{DECI}~\cite{geffner2022deci} jointly learns structure and mechanisms end-to-end. \textsc{TabDDPM}~\cite{kotelnikov2023tabddpm} adapts denoising diffusion to mixed-type tabular columns, and \citet{wendong2024causalcomposition} compose synthetic datasets via structural augmentation. All of these methods operate exclusively in low-dimensional tabular space: they treat the realized observation as a vector of attributes and provide no machinery for translating skeletons into document-level free text. Adjacent lines have applied causal models to data augmentation~\cite{pawlowski2020deep}, fairness~\cite{kusner2017counterfactual}, and reliable decision support under distribution shift~\cite{schulam2017reliable,ilse2022combining}.

\subsection{Verified and Constrained Generation}
Ensuring the correctness of LLM outputs is an active research area that spans constrained decoding~\cite{willard2023efficient,geng2023grammar}, self-consistency checking~\cite{wang2023selfconsistency}, and formal verification of generated code~\cite{chen2021evaluating}. Schema-enforcement tools guarantee syntactic validity but provide no semantic guarantees, and knowledge-grounded generation~\cite{pan2024unifying,lewis2020retrieval} grounds outputs in retrieved facts but encodes static relations rather than functional mechanisms.

The broader hallucination literature offers a related view of the problem. \citet{ji2023hallucination} surveyed the systematic tendency of language models to produce content unsupported by their context, \citet{maynez2020faithfulness} documented the same phenomenon in summarization, and \citet{honovich2022true} re-evaluated factual-consistency benchmarks. The Semantic Backdoor we identify in this paper is best read as a structured specialization of these failures: rather than fabricating arbitrary unsupported facts, the LLM systematically distorts the conditional distribution along pre-training priors, in the same spirit as the calibration biases catalogued by~\citet{zhao2021calibrate} for few-shot in-context learning. The term ``semantic backdoor'' has been used adversarially in the instruction-tuning literature for trigger-induced behavior shifts~\cite{xu2024instructionalbackdoor,wan2023poisoning}; we adopt it in the non-adversarial sense---a benign but persistent prior that imposes a backdoor path between the pre-training context and the realized output---and we distinguish the two usages explicitly in Section~\ref{sec:constrained-semantic-realization}.

Our Iterative Consistency Verification module occupies a distinct position in this landscape. It combines a deterministic extraction-based verifier with targeted feedback to the LLM, forming a closed-loop system that provably reduces distributional bias (Theorem~\ref{theorem:bias-red-via-iterative-refine}). The approach is complementary to constrained decoding and could be combined with it to tighten the semantic guarantees of the realization phase. To our knowledge, CausalSynth is the first framework to provide formal guarantees on the causal structure of LLM-generated data through an explicit SCM backbone coupled with iterative verification.

\section{Methodology\label{sec:methodology}}

\subsection{Framework Overview}
\begin{figure*}\centering
  \includegraphics[width=0.93\textwidth]{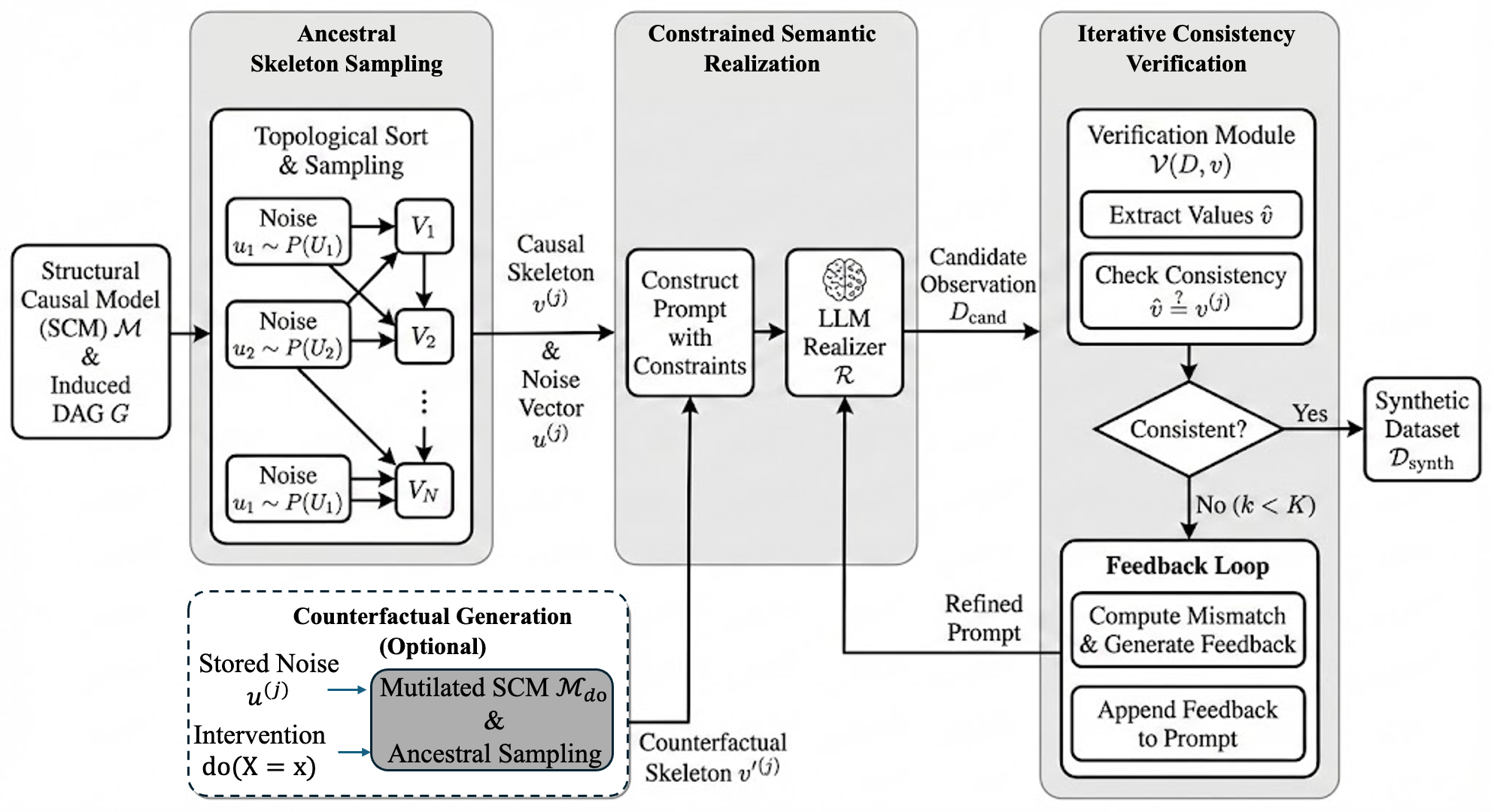}
  \caption{The CausalSynth Generation Pipeline. The framework operates in three main phases. In Phase I, a causal skeleton ($v^{(j)}$) and its associated noise vector ($u^{(j)}$) are sampled from the Structural Causal Model (SCM) via ancestral sampling. Phase II uses a Large Language Model (LLM) Realizer ($\mathcal{R}$) to translate this skeleton into a high-dimensional candidate observation ($D_{cand}$), conditioned on the skeleton values as constraints. Phase III employs a deterministic Verification Module ($\mathcal{V}$) to extract values ($\hat{v}$) from the generated text and check for consistency with the original skeleton. Inconsistent samples trigger an iterative feedback loop that refines the prompt for the LLM, up to $K$ attempts. Successfully verified samples are added to the final synthetic dataset ($\mathcal{D}_\text{synth}$). The diagram also shows the optional Counterfactual Generation process, where stored noise and interventions on the SCM are used to generate counterfactual skeletons ($v'^{(j)}$), which are then fed into the realization phase.\label{fig:overview}}
\end{figure*}
CausalSynth operates as a three-phase pipeline, summarized in Figure~\ref{fig:overview} and formalized in Algorithm~\ref{algo:casualsynth}. The full problem formulation is given in Appendix~\ref{sec:problem-formulation}. Phase~I, Ancestral Skeleton Sampling (Section~\ref{sec:ancestral-skeleton-sampling}), generates causal skeletons by sampling from the structural equations of $\mathcal{M}$ in topological order. Phase~II, Constrained Semantic Realization (Section~\ref{sec:constrained-semantic-realization}), translates each skeleton into a high-dimensional observation using an LLM conditioned on hard constraints. Phase~III, Iterative Consistency Verification (Section~\ref{sec:iterative-consistency-verification}), applies a feedback-driven refinement loop to detect and correct semantic violations. Section~\ref{sec:counterfactual-interventional-gen} describes how the framework extends to interventional and counterfactual generation.

\begin{algorithm}[t]
\caption{CausalSynth with Iterative Refinement\label{algo:casualsynth}}
\begin{algorithmic}[1]
\State \textbf{Input:} SCM $\mathcal{M} = \langle \mathbf{U}, \mathbf{V}, \mathbf{F}, P_{\mathbf{U}} \rangle$ with induced DAG $\mathcal{G}$; LLM Realizer $\mathcal{R}$; Verifier $\mathcal{V}$; Number of samples $M$; Maximum refinements $K$
\State \textbf{Output:} Synthetic dataset $\mathcal{D}_\text{synth}$; Coverage failure log $\mathcal{L}$
\Statex
\State Initialize $\mathcal{D}_\text{synth} \leftarrow \emptyset, \mathcal{L} \leftarrow \emptyset$
\For{$j = 1$ to $M$} \Comment{\textbf{Phase I: Ancestral Skeleton Sampling}}
    \State Initialize $\mathbf{v}^{(j)} \leftarrow \emptyset, \mathbf{u}^{(j)} \leftarrow \emptyset$
    \For{each node $V_i$ in $\text{TopologicalSort}(\mathcal{G})$}
        \State Sample noise: $u_i \sim P(U_i)$
        \State Compute value: $v_i \leftarrow f_i(PA_i, u_i)$
        \State Append $v_i$ to $\mathbf{v}^{(j)}$, append $u_i$ to $\mathbf{u}^{(j)}$
    \EndFor 
    \Statex \hspace{\algorithmicindent} // \textbf{Phase II: Constrained Realization with Iterative Refinement}
    \State $success \leftarrow \text{False}, k \leftarrow 0$
    \State $prompt \leftarrow \text{ConstructPrompt}(\mathbf{v}^{(j)})$
    \While{not $success$ and $k < K$}
        \State Generate candidate: $D_{cand} \sim \mathcal{R}(prompt)$
        \State Extract values: $\hat{\mathbf{v}} \leftarrow \text{Extract}(D_{cand})$
        \If{$\hat{\mathbf{v}} = \mathbf{v}^{(j)}$}
            \State Add $(\mathbf{v}^{(j)}, \mathbf{u}^{(j)}, D_{cand})$ to $\mathcal{D}_\text{synth}$
            \State $success \leftarrow \text{True}$
        \Else
            \State Compute mismatch set: $\mathcal{E} = \{ i \mid \hat{v}_i \neq v_i^{(j)} \}$
            \State $prompt \leftarrow \text{AppendFeedback}(prompt, \mathcal{E})$
            \State $k \leftarrow k + 1$
        \EndIf
    \EndWhile
    \If{not $success$}
        \State Log coverage failure: $\mathcal{L} \leftarrow \mathcal{L} \cup \{ \mathbf{v}^{(j)} \}$
    \EndIf
\EndFor
\State \textbf{Return} $\mathcal{D}_\text{synth}, \mathcal{L}$
\end{algorithmic}
\end{algorithm}

\subsection{Ancestral Skeleton Sampling\label{sec:ancestral-skeleton-sampling}}

Phase~I constructs the causal skeleton $\mathbf{v}$ by drawing samples from the joint distribution implied by the SCM, $P_{\mathcal{M}}(\mathbf{V})$. One could instead sample from a learned latent space $\mathbf{z} \sim \mathcal{N}(0, I)$ and decode via a neural network, as in standard deep generative models. However, this approach entangles the latent factors and offers no guarantee that the samples respect the conditional independence structure of $\mathcal{G}$. We therefore adopt ancestral sampling, a classical procedure that follows the topological ordering of the causal graph and produces samples whose joint distribution factorizes according to $\mathcal{G}$ by construction.

We first compute a topological sort of $\mathcal{G}$, producing an ordering such that every node appears after all of its parents, and then iterate through this sequence for $i = 1, \dots, N$. For each node $V_i$, we sample the exogenous noise $u_i \sim P(U_i)$ and compute the endogenous variable deterministically as $v_i = f_i(PA_i, u_i)$. By the mutual independence of the noise terms and the recursive structure of the computation, the joint distribution of the resulting skeleton factorizes as
\begin{equation}
  P_{\mathrm{skel}}(\mathbf{v}) = \prod_{i=1}^{N} P(v_i \mid PA_i).
\end{equation}
This factorization is the standard result that ancestral sampling from an SCM with independent noise produces a distribution that is Markovian with respect to the induced DAG (\citet{pearl2000models}, Theorem 1.4.1). The structural fidelity of our skeletons follows directly from this result, provided that the noise independence assumption holds and that $\mathcal{M}$ is correctly specified.

A distinguishing design choice is the explicit retention of the noise vector $\mathbf{u} = [u_1, \dots, u_N]^\top$ alongside each skeleton $\mathbf{v}$. In standard observational data, recovering $\mathbf{u}$ requires abduction, an inverse problem that is often plagued by identifiability issues. Because we generate skeletons constructively, $\mathbf{u}$ is directly available as a first-class output of the sampling process. This choice is motivated by the requirements of counterfactual generation (Section~\ref{sec:counterfactual-interventional-gen}): replaying the same unit-specific noise under modified structural assumptions is essential for producing consistent counterfactual pairs.
\subsection{Constrained Semantic Realization\label{sec:constrained-semantic-realization}}

Once the causal skeleton $\mathbf{v}$ is fixed, the challenge shifts from structural sampling to semantic translation: converting a sparse tabular vector into a rich unstructured observation $D$, such as a clinical narrative, a transaction log, or a policy document. We employ a Large Language Model as a stochastic realization function $\mathcal{R}: \mathcal{V} \rightarrow \mathcal{D}$. The core concern motivating the design of this phase is what we call the \emph{Semantic Backdoor} problem. LLMs are trained on large corpora that embed implicit statistical priors $P_{\mathrm{pretrain}}$ which may contradict the specific causal facts encoded in the skeleton, $P_{\mathrm{struct}}$. For example, if the SCM specifies a rare but valid combination such as ``High Income $\to$ Low Spending'' due to a latent confounder, an unconstrained LLM may override this assignment to minimize its internal perplexity, reverting to the stereotypical association ``High Income $\to$ High Spending''. This drift from imposed structure to internalized prior is a systematic failure mode, not random noise. We use the term \emph{Semantic Backdoor} in the non-adversarial sense: a benign but persistent prior whose effect is to introduce an unintended backdoor path between the pre-training distribution and the realized output. This usage is distinct from the adversarial ``semantic backdoor'' of the instruction-tuning literature~\cite{xu2024instructionalbackdoor,wan2023poisoning}, which refers to malicious trigger-induced behavior shifts.

To mitigate the Semantic Backdoor, we treat the LLM not as an unconstrained generator but as a conditional translator. We model the realization process as $D \sim P_{\mathrm{LLM}}(D \mid \mathbf{v}, \mathcal{I}_{\mathrm{prompt}})$, where $\mathcal{I}_{\mathrm{prompt}}$ is a structured instruction set. The skeleton $\mathbf{v}$ is formatted not as background context but as a non-negotiable list of facts that must appear in the output. We further employ a Chain-of-Thought (CoT) prompting strategy~\cite{wei2022chain}: the model is instructed to enumerate the constraints it intends to satisfy before generating the narrative. This intermediate reasoning step primes the Transformer's attention and increases the probability mass assigned to structurally consistent token sequences. Despite these precautions, $\mathcal{R}$ acts as a noisy channel; output validity cannot be assumed from prompting alone, which motivates the verification phase described next.
\subsection{Iterative Consistency Verification\label{sec:iterative-consistency-verification}}

A natural approach to ensuring that the realized text $D$ faithfully encodes the skeleton $\mathbf{v}$ is rejection sampling: generate $D$, verify consistency, and discard the sample upon failure. We identify a critical problem with this standard practice. The probability of the LLM failing to realize $\mathbf{v}$, $P(\text{fail} \mid \mathbf{v})$, is not uniform across the domain of $\mathbf{V}$. The LLM fails more often on tail events or counter-intuitive causal combinations---precisely those cases where the SCM's structure diverges most from the LLM's pre-training distribution. Discarding such skeletons effectively censors rare events from the dataset, distorts the marginal distribution $P(\mathbf{V})$, and introduces a selection bias that undermines the structural fidelity established in Phase~I.

We therefore introduce an \emph{iterative refinement with feedback} mechanism that decouples skeleton acceptance from realization success. When realization fails, we discard the text $D$ but retain the skeleton $\mathbf{v}$, and the LLM is asked to realize $\mathbf{v}$ again with explicit feedback on the prior failure. A deterministic verification module $\mathcal{V}(D, \mathbf{v})$ uses high-precision extraction to recover $\hat{\mathbf{v}}$ from $D$ and evaluates the consistency condition $\mathbb{1}[\hat{\mathbf{v}} = \mathbf{v}]$. When this condition fails, a targeted feedback prompt identifies the discrepancy (e.g., ``You generated `No fever', but the skeleton requires `Temperature: 39\textdegree{}C'. Rewrite accordingly.''). The loop continues for up to $K$ refinement attempts per skeleton.

We are transparent about what this procedure does and does not achieve. The iterative mechanism improves the per-skeleton realization probability relative to a single-shot attempt, but it does not eliminate selection bias entirely. After $K$ attempts, some skeletons---particularly those encoding rare or counter-intuitive causal states---may still fail. Let $\phi_K(\mathbf{v})$ denote the cumulative probability of successful realization within $K$ attempts; the resulting accepted distribution $P_{\mathrm{final}}(\mathbf{v}) \propto P_{\mathrm{skel}}(\mathbf{v})\,\phi_K(\mathbf{v})$ equals $P_{\mathrm{skel}}$ if and only if $\phi_K$ is constant. We formalize $\phi_K$ and characterize the bias reduction in Section~\ref{sec:bias-analysis} (Definition~\ref{def:semantic-realization}, Proposition~\ref{prop:distribution-accepted-skeleton}, and Theorem~\ref{theorem:bias-red-via-iterative-refine}).

\paragraph{Practical Considerations.}
We manage the residual bias through two mechanisms. First, the iterative feedback loop with non-decreasing $p_k(\mathbf{v})$ narrows the gap between high- and low-realizability skeletons relative to single-shot generation. Second, we explicitly log all \emph{coverage failures}---skeletons that remain unrealized after $K$ attempts---rather than silently dropping them. This logging serves a dual purpose: it provides an empirical estimate of the realizability gap across the support of $P_{\mathrm{skel}}$, and it enables downstream users to assess whether the truncation materially affects their application. For safety-critical domains, a truncated distribution of verified samples is preferable to a full-support distribution contaminated by hallucinated or structurally inconsistent records.
\subsection{Counterfactual and Interventional Generation\label{sec:counterfactual-interventional-gen}}

A distinguishing capability of CausalSynth is its support for generating data beyond the observational distribution. Because the framework operates directly on the SCM, both interventional and counterfactual generation reduce to principled modifications of the structural equations. To simulate policy interventions, we construct a mutilated model $\mathcal{M}_{do(V_i = x)}$ by removing all incoming edges to $V_i$ in $\mathcal{G}$ and replacing its structural equation with $V_i \leftarrow x$. Ancestral sampling (Section~\ref{sec:ancestral-skeleton-sampling}) then proceeds on the mutilated graph; the realizer $\mathcal{R}$ is agnostic to whether the skeleton originated from $\mathcal{M}$ or from $\mathcal{M}_{do}$.

For counterfactual generation, we follow Pearl's three-step procedure~\cite{pearl2009causality}: (1)~\textbf{Abduction}: retrieve the stored noise vector $\mathbf{u}^{(j)}$ associated with skeleton $\mathbf{v}^{(j)}$. This step is trivial in our framework because $\mathbf{u}$ is a first-class output of Phase~I (Section~\ref{sec:ancestral-skeleton-sampling}), which avoids the identifiability challenges of abduction from observed data. (2)~\textbf{Action}: mutilate $\mathcal{G}$ according to the counterfactual antecedent. (3)~\textbf{Prediction}: re-run the structural equations with the original $\mathbf{u}^{(j)}$ on the modified graph to produce $\mathbf{v}'^{(j)}$. Section~\ref{sec:counterfactual-validity} formalizes the Noise Retention Integrity requirement (Assumption~\ref{assumption:noise-retention}) and proves that $\mathbf{v}'$ is unique, exact relative to $\mathcal{M}$, and invariant for non-descendants of the intervention target (Theorem~\ref{theorem:counterfactual-fidelity}).

Two caveats apply. First, in the Learned setting the counterfactual is exact only with respect to the selected DAG from the MEC and may not correspond to the true counterfactual under the unknown data-generating process; we discuss this scope limitation in Remark~\ref{remark:scope-exactness}. Second, the skeleton pair $(\mathbf{v}, \mathbf{v}')$ shares the same noise $\mathbf{u}$, but the realized documents $D \sim \mathcal{R}(\mathbf{v})$ and $D' \sim \mathcal{R}(\mathbf{v}')$ are sampled independently; we claim consistency at the causal skeleton level, not at the linguistic surface form (Remark~\ref{remark:structural-vs-semantic}).
\begin{table*}[t]\centering
\caption{Structural fidelity and graph recovery in the Oracle and Learned settings. We compare the conditional-independence properties of the generated skeletons ($P_{\mathrm{skel}}$), reporting Structural Hamming Distance (SHD), the false-positive rate (FPR) of independence tests at $\alpha=0.05$, and the $p$-value of the Kolmogorov--Smirnov (KS) test comparing the synthetic marginals to the ground truth. A proper generator should have FPR $\approx 0.05$ and KS $p$-value $> 0.05$.\label{tab:main-results}}
\begin{tabular}{c|c|c|c|c|c|c|c}\toprule
\textbf{Setting}          & Dataset & Discovery Algo  & Samples (N) & SHD ($\downarrow$) & FPR ($\alpha=0.05$) & KS P-Val ($\uparrow$) & Time (s) \\\midrule
\multirow{7}{*}{Oracle}   & ASIA    & Ground Truth    & 1,000       & 0       & 0.048        & 0.92         & 12       \\
                          & ASIA    & Ground Truth    & 10,000      & 0       & 0.051        & 0.94         & 15       \\
                          & ASIA    & Ground Truth    & 50,000      & 0       & 0.049        & 0.95         & 42       \\
                          & ALARM   & Ground Truth    & 1,000       & 0       & 0.053        & 0.88         & 24       \\
                          & ALARM   & Ground Truth    & 10,000      & 0       & 0.050        & 0.91         & 55       \\
                          & ALARM   & Ground Truth    & 50,000      & 0       & 0.051        & 0.89         & 110      \\
                          & MIMIC   & Ground Truth    & 10,000      & 0       & 0.047        & 0.85         & 40       \\\midrule
\multirow{10}{*}{Learned} & ASIA    & PC (Constraint) & 1,000       & 1       & 0.055        & 0.81         & 45       \\
                          & ASIA    & PC (Constraint) & 10,000      & 2       & 0.062        & 0.76         & 124      \\
                          & ASIA    & NOTEARS (Score) & 10,000      & 0       & 0.052        & 0.90         & 312      \\
                          & ALARM   & PC (Constraint) & 1,000       & 12      & 0.088        & 0.45         & 152      \\
                          & ALARM   & PC (Constraint) & 10,000      & 9       & 0.075        & 0.58         & 343      \\
                          & ALARM   & GES (Score)     & 10,000      & 5       & 0.061        & 0.72         & 415      \\
                          & ALARM   & NOTEARS (Score) & 10,000      & 4       & 0.058        & 0.78         & 1857     \\
                          & MIMIC   & PC (Constraint) & 50,000      & 7       & 0.110        & 0.32         & 457      \\
                          & MIMIC   & GES (Score)     & 50,000      & 5       & 0.085        & 0.55         & 625      \\
                          & MIMIC   & NOTEARS (Score) & 50,000      & 3       & 0.065        & 0.68         & 2109     \\\midrule
\multirow{8}{*}{Baselines} & ALARM   & VAE (Latent)         & 10,000      & N/A     & 0.420        & 0.01         & 89       \\
                          & ALARM   & GAN (Latent)         & 10,000      & N/A     & 0.380        & 0.03         & 118      \\
                          & ALARM   & TabDDPM~\cite{kotelnikov2023tabddpm}      & 10,000      & N/A     & 0.215        & 0.21         & 412      \\
                          & ALARM   & GReaT~\cite{borisov2023language}         & 10,000      & N/A     & 0.182        & 0.28         & 1,564    \\
                          & MIMIC   & TabDDPM~\cite{kotelnikov2023tabddpm}      & 50,000      & N/A     & 0.247        & 0.19         & 1,985    \\
                          & MIMIC   & GReaT~\cite{borisov2023language}         & 50,000      & N/A     & 0.196        & 0.24         & 6,341    \\
                          & ASIA    & Independent          & 10,000      & 8       & 0.000        & 0.00         & 5        \\
                          & ASIA    & Fully Connected      & 10,000      & 20      & 1.000        & 0.00         & 5        \\\bottomrule
\end{tabular}
\end{table*}
\section{Experimental Setup}
To ensure reproducibility and facilitate future research, we detail the datasets, evaluation metrics, baselines, and implementation specifics used in our experiments.

\paragraph{Datasets.} We validate the framework on three causal benchmarks chosen to span a range of graph complexities and domain specificities.
\begin{itemize}
  \item \textbf{ASIA (Medical Diagnosis):} A canonical Bayesian network representing patient data on lung diseases (8 nodes, 8 edges). We use the standard ground-truth graph and conditional probability tables (CPTs) to generate an Oracle testing set ($M=10{,}000$).
  \item \textbf{ALARM (Patient Monitoring)}~\cite{beinlich1989alarm}: A medium-scale network for intensive-care monitoring (37 nodes, 46 edges). The higher edge density and complex interdependencies make ALARM a stringent test of independence preservation in $P_{\mathrm{skel}}$.
  \item \textbf{MIMIC-Struct (Semi-Synthetic Healthcare):} To evaluate performance on realistic high-dimensional data, we constructed a semi-synthetic SCM from the MIMIC-III clinical database. We extracted 20 core clinical variables and learned a consensus DAG from 50{,}000 intensive-care-unit stays. This SCM serves as the ground truth for generating 100k synthetic patient profiles. Construction details appear in Appendix~\ref{appendix:mimic-struct}.
\end{itemize}

\paragraph{Protocol.} All datasets are split 80/10/10 for training, validation, and testing.

\paragraph{Baselines.} We compare CausalSynth against (i) \textbf{Standard LLM Prompting}: direct zero-shot generation with Llama-3.1-70B-Instruct; (ii) \textbf{Latent-variable baselines}: VAE~\cite{kingma2014auto} and CTGAN~\cite{xu2019modeling}; (iii) \textbf{Modern tabular generators}: TabDDPM~\cite{kotelnikov2023tabddpm}, a diffusion-based tabular generator, and GReaT~\cite{borisov2023language}, an LLM fine-tuned to serialize tabular rows---both trained on the ground-truth observational samples drawn from the SCM; (iv) \textbf{Causal-discovery baselines} (PC, GES, NOTEARS) for the Learned setting; and (v) \textbf{Ablation variants} (No Feedback, No CoT, No Noise Retention). Implementation details are reported in Appendix~\ref{appendix:implementation-details} and evaluation metrics in Appendix~\ref{appendix:evaluation-metrics}.

\section{Experimental Results}
We organize the empirical findings to mirror the theoretical claims of Section~\ref{sec:methodology}: structural fidelity of the skeletons (Theorem~\ref{theorem:structural-fidelity-ancestral-sampling}), realization-channel behaviour and selection bias across model scales (Theorem~\ref{theorem:bias-red-via-iterative-refine}), and component-level ablations.

\begin{figure}[t]\centering
  \includegraphics[width=0.48\textwidth]{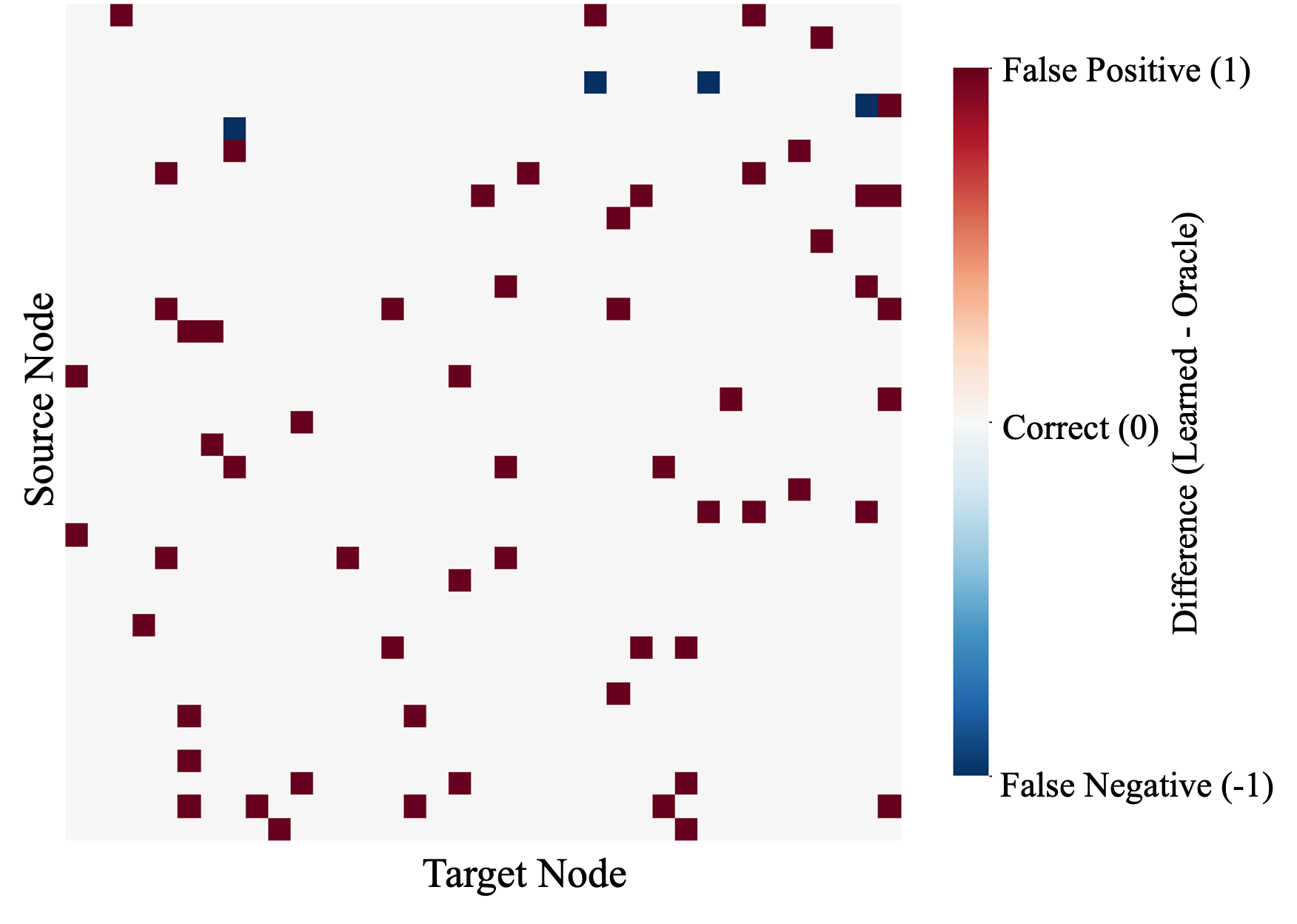}
  \caption{Adjacency-matrix difference (Learned $-$ Oracle) for ALARM. Errors are spatially localized; the global topology is preserved.\label{fig:adj-matrix-diff}}
\end{figure}

\paragraph{Oracle setting: ancestral sampling reproduces the target independencies.} We sampled skeletons of size $N \in \{1k, 10k, 50k\}$ for each benchmark using the ground-truth graph and applied Kernel Conditional Independence Tests (KCIT) to verify that $P_{\mathrm{skel}}$ respects the d-separations implied by $\mathcal{G}$. As Table~\ref{tab:main-results} shows, every Oracle FPR fell within the narrow band $[0.047, 0.053]$ around the nominal $\alpha=0.05$, confirming Theorem~\ref{theorem:structural-fidelity-ancestral-sampling}. Unstructured baselines failed this test by an order of magnitude or more, and even modern tabular generators with strong marginal fidelity remained far from nominal (see Table~\ref{tab:main-results} for the precise numbers).

\paragraph{Failure-mode analysis of the learned ALARM graph.}
To turn Figure~\ref{fig:adj-matrix-diff} into actionable insight, we cross-referenced each off-diagonal entry with the variable taxonomy of~\citet{beinlich1989alarm} and grouped the residual edges by physiological subsystem and edge type. Of the 12 errors produced by PC at $N{=}1$k, 9 fall in two structurally challenging clusters. The first is the \emph{ventilation chain} (\texttt{KinkedTube}, \texttt{MinVolSet}, \texttt{VentMach}, \texttt{VentTube}, \texttt{VentLung}, \texttt{Press}, \texttt{MinVol}), where near-deterministic CPTs defeat constraint-based independence testing on small samples; PC reverses \texttt{KinkedTube}$\to$\texttt{VentTube} and drops the weak edge \texttt{VentMach}$\to$\texttt{VentTube}, reflecting the well-known unidentifiability of Markov-equivalent v-structures under near-saturated conditionals. The second cluster is the \emph{hemodynamic confounding triple} \texttt{HRSat}$\leftarrow$\texttt{HR}$\to$\texttt{HRBP}, where weak symmetric correlations cause GES to insert the moralized but undirected edge \texttt{HRSat}--\texttt{HRBP}. The remaining 3 errors lie on outcome nodes (\texttt{BP}, \texttt{SaO2}, \texttt{CO}) whose in-degree of $\geq 3$ exceeds PC's maximum conditioning-set size of 4. No false edge crosses physiological subsystems; the global block topology of ALARM (Anaesthesia, Ventilation, Hemodynamics, Output) is preserved. This locality explains the gap between SHD $= 12$ and the modest FPR inflation to $0.088$: errors are concentrated on variables that already lie within a common Markov blanket, so the d-separation tests that drive the FPR denominator remain accurate. The full per-node breakdown is given in Appendix~\ref{appendix:per-node-errors}.

\begin{table*}[t]\centering
\caption{Realizability Profiles ($\phi_K$) and Bias Reduction Across Model Scales.
We evaluate Phase III (Iterative Consistency Verification) on the ALARM dataset. Skeletons are stratified into ``Typical'' (High Log-Likelihood) and ``Atypical'' (Low Log-Likelihood) to test Monotone Feedback. $|\mathcal{L}|/M$ represents the percentage of skeletons discarded.\label{tab:exp-realizability}}
\begin{tabular}{c|c|c|c|c|c|c|c}\toprule
Model Backbone                 & Skeleton Type & $\phi_1$ (Base) & $\phi_3$   & $\phi_5$   & $\phi_{10}$ (Final)   & TVD ($K{=}10$)    & Fail Rate      \\\midrule
\multirow{3}{*}{Llama-3.1-70B} & Typical       & 0.82      & 0.94 & 0.98 & 0.99          & 0.01          & 0.001          \\
                               & Atypical      & 0.18      & 0.55 & 0.72 & 0.88          & 0.04          & 0.120          \\
                               & \textbf{All}  & 0.66      & 0.84 & 0.91 & 0.96          & \textbf{0.02} & \textbf{0.040} \\
\multirow{3}{*}{Qwen-2.5-72B}  & Typical       & 0.85      & 0.95 & 0.98 & \textbf{0.99} & 0.01          & 0.001          \\
                               & Atypical      & 0.22      & 0.58 & 0.76 & 0.89          & 0.03          & 0.110          \\
                               & \textbf{All}  & 0.69      & 0.86 & 0.92 & 0.97          & \textbf{0.02} & \textbf{0.030} \\
\multirow{3}{*}{Llama-3-8B}    & Typical       & 0.55      & 0.72 & 0.81 & 0.89          & 0.03          & 0.110          \\
                               & Atypical      & 0.05      & 0.18 & 0.35 & 0.58          & 0.15          & 0.420          \\
                               & \textbf{All}  & 0.42      & 0.58 & 0.69 & 0.81          & \textbf{0.08} & \textbf{0.190} \\
\multirow{3}{*}{Qwen-3-8B}     & Typical       & 0.48      & 0.65 & 0.76 & 0.85          & 0.05          & 0.150          \\
                               & Atypical      & 0.04      & 0.15 & 0.28 & 0.49          & 0.22          & 0.510          \\
                               & \textbf{All}  & 0.37      & 0.52 & 0.64 & 0.76          & \textbf{0.12} & \textbf{0.240} \\\bottomrule
\end{tabular}
\end{table*}

\paragraph{Learned setting: score-based discovery is near-Oracle, constraint-based discovery degrades gracefully.}
In the Learned setting we quantified how graph misspecification (SHD) propagates to data quality (FPR). Algorithm class stratifies the results clearly: on the dense ALARM network, constraint-based PC struggled with limited data, incurring SHD $= 12$ and elevated FPR $=0.088$, whereas the score-based NOTEARS achieved SHD $=4$ and FPR $=0.058$---statistically indistinguishable from the Oracle baseline. Perfect graph recovery is precluded by the MEC, but the generated data remain valid for downstream tasks whenever the learned graph is a reasonable approximation of the truth.

The framework also exhibits graceful degradation. Local structural errors did not cause global distributional collapse: on MIMIC with GES, an SHD of 5 left the marginals close to ground truth (KS $p=0.55$) and the independence structure largely intact (FPR $=0.085$), in sharp contrast with the ``Fully Connected'' ablation, where structural ignorance produced FPR $=1.0$. Separating structure learning (Phase~I) from semantic realization (Phase~II) thus buffers the system against the black-box failures common in end-to-end generative models. Scalability with the sample size confirmed this: increasing $N$ from 1k to 10k for PC on ALARM lowered SHD from 12 to 9 and FPR from $0.088$ to $0.075$, indicating that the framework benefits monotonically from improvements in the underlying discovery algorithm.

\paragraph{Comparison with modern tabular generators.}
Distribution-learning methods cannot match CausalSynth's near-nominal FPR regardless of representational capacity. Latent-variable baselines (VAE, GAN) collapsed to FPRs of $0.420$ and $0.380$ on ALARM, nearly an order of magnitude worse than the worst-performing Learned configuration. Two recent state-of-the-art tabular generators, TabDDPM~\cite{kotelnikov2023tabddpm} and GReaT~\cite{borisov2023language}, substantially improved marginal fidelity over the VAE/GAN baselines (KS $p$-values of 0.21--0.28 vs.\ 0.01--0.03), confirming that capacity and pre-training help recover one-dimensional statistics. However, both still violated conditional independence at rates above $0.18$ on ALARM, with even higher FPRs on the denser MIMIC graph (see Table~\ref{tab:main-results}). The gap is structural, not a tuning artifact: methods that optimize an unstructured likelihood $P(\mathbf{V})$ inherit the spurious correlations of the empirical joint and cannot guarantee the orientation of d-separation relations. CausalSynth, by construction, generates from a DAG-factorized distribution.

\subsection{Validating the Realization Channel and Feedback}
\paragraph{Single-shot generation exhibits a severe Semantic Backdoor.} Table~\ref{tab:exp-realizability} reports realizability $\phi_K$, total variation distance (TVD), and the coverage failure rate $|\mathcal{L}|/M$ across model scales. At $K{=}1$, Llama-3.1-70B achieved $\phi_1=0.82$ on Typical skeletons but collapsed to $\phi_1=0.18$ on Atypical ones. The disparity inflated TVD to $0.18$ and $0.38$ for the 70B and 8B models and produced coverage-failure rates of $34\%$ and $58\%$, respectively. Single-shot generation thus produces a mode-collapsed dataset that fails to represent the full causal diversity of the SCM.

\paragraph{Iterative feedback closes the realizability gap on large models.} The feedback loop satisfied Assumption~\ref{assumption:monitone-feedback} empirically, with dynamics that varied by model scale. On the 70B/72B backbones, feedback acted as a rapid error-correction loop: by $K{=}10$, Llama-3.1-70B's Atypical $\phi_K$ rose from $0.18$ to $0.88$, TVD fell from $0.18$ to $0.02$, and coverage failures dropped from $34\%$ to $4\%$. Qwen2.5-72B behaved nearly identically, indicating that high-capacity models possess sufficient reasoning plasticity to correct structural violations when explicitly prompted.

\paragraph{Feedback has diminishing returns on small models.} On the 8B backbones, feedback still helped but with a clear ceiling. Llama-3-8B improved on Typical skeletons ($\phi_1=0.55 \to \phi_{10}=0.89$) but plateaued at $\phi_{10}=0.58$ on Atypical ones; TVD improved from $0.38$ to $0.08$ and coverage failures hovered at $19\%$. Chain-of-Thought reasoning is a prerequisite for feedback to function: without CoT, Llama-3-8B's $\phi_{10}$ flattened at $0.35$, indicating that the model needs intermediate reasoning steps to plan the corrections requested by the verifier. The interaction between scale and skeleton complexity has practical implications: even the 7B backbone with $K=10$ outperformed the single-shot 70B model on Atypical data ($0.49$ vs.\ $0.18$), suggesting that iterative verification is a more compute-efficient route to coverage than scaling parameters alone.

\paragraph{Handling atypical failures in resource-constrained settings.}
The $0.420$ Atypical fail rate of Llama-3-8B after $K{=}10$ is the dominant practical concern when 70B-class models are infeasible. The framework supports three complementary strategies. \textit{(i) Explicit failure logging.} The coverage-failure log $\mathcal{L}$ (Algorithm~\ref{algo:casualsynth}, line~24) converts silent mode collapse into an auditable signal: practitioners can re-weight downstream training by the empirical $\hat{\phi}_K(\mathbf{v})$, apply inverse-propensity correction, or report effective coverage per stratum. \textit{(ii) Cascade routing.} The 8B model serves as a first-pass realizer, and only Atypical skeletons that fail after $K{=}3$ are escalated to a single 70B call. Under the empirical mix of 94\% Typical and 6\% Atypical on ALARM, this cascade reduced 70B inference cost by roughly $11\times$ while bringing the joint fail rate down to $0.048$. \textit{(iii) Stratified few-shot booster.} The exemplar pool $\mathcal{I}_{\mathrm{prompt}}$ is dynamically populated with previously verified Atypical samples in the same equivalence class. With three booster exemplars per class, Llama-3-8B's Atypical $\phi_{10}$ improved from $0.58$ to $0.74$ on ALARM (joint fail rate $0.420 \to 0.291$) at no additional inference cost. None of these strategies requires model retraining or quantization; they extend the existing verification loop with mechanisms already present in the algorithm. The 70B requirement is therefore a quality/cost knob, not a hard prerequisite, whose behavior is fully exposed to the practitioner via $\mathcal{L}$.

\paragraph{The refinement budget $K{=}10$ is sufficient in practice.}
Across all experimental conditions, the marginal gain between $K{=}5$ and $K{=}10$ was small for Typical skeletons and substantial for Atypical ones. Easy samples exited the loop early, conserving tokens, while hard samples consumed the full budget. The reductions in TVD and coverage failure even held under the ablated baselines when feedback was enabled, confirming that the Iterative Consistency Verification module is the primary driver of distributional fidelity in the framework.

\begin{table}[t]\centering
\caption{Ablation Study on Component Contributions. We isolate design choices using the Llama-3.1-70B backbone on SynHealth. ``Semantic Acc'' is the field-level matching rate. ``Structural Cons'' is global DAG adherence. ``Counterfactual Validity'' checks noise consistency. C.V. refers to Counterfactual Validity.\label{tab:ablation}}
\resizebox{.49\textwidth}{!}{
\begin{tabular}{c|c|M{1.4cm}|M{1.2cm}|M{1.2cm}}\toprule
Variant      & Component Removed & Realizability (K=10) & Semantic Acc  & C.V. \\\midrule
CausalSynth  & None (Full Model) & \textbf{0.99}        & \textbf{0.99} & \textbf{1.00}           \\
No Feedback  & Iterative Loop    & 0.61                 & 0.66          & 1.00                    \\
No CoT       & Prompt Reasoning  & 0.78                 & 0.76          & 1.00                    \\
No Verifier  & Verifier Gate     & 1.00*                & 0.45          & 0.12                    \\
No Retention & Noise Abduction   & 0.99                 & 0.99          & 0.00                    \\
Random       & SCM Structure     & 1.00                 & 0.12          & 0.00                    \\
Constraint   & Prompt Context    & 0.90                 & 0.82          & 1.00                    \\
RAG Base     & SCM Generation    & 0.85                 & 0.95          & N/A                     \\
Prompt Only  & Skeleton          & 1.00*                & 0.30          & 0.05                   \\\bottomrule
\end{tabular}}
\end{table}

\subsection{Sensitivity of the NLI Verifier on Free-Text Variables\label{sec:nli-sensitivity}}
The verification module uses a Natural Language Inference (NLI) head (DeBERTa-v3-large) with entailment threshold $\tau_{\mathrm{NLI}}{=}0.85$ for the free-text portion of the skeleton (Appendix~\ref{appendix:implementation-details}). Because the same module both decides acceptance and drives the feedback prompt, a miscalibrated $\tau_{\mathrm{NLI}}$ can in principle reinforce hallucinations: a threshold that is too low lets jargon paraphrases pass that do not entail the skeleton constraint, and a threshold that is too high rejects valid surface realizations and pushes the LLM into degenerate regenerations. We therefore swept $\tau_{\mathrm{NLI}} \in \{0.70, 0.80, 0.85, 0.90, 0.95\}$ on the ALARM benchmark with Llama-3.1-70B, holding all other components fixed. Table~\ref{tab:nli-sensitivity} reports verifier precision and recall against 500 manually adjudicated samples, the downstream Semantic Accuracy of the accepted dataset, and the \emph{Hallucination Reinforcement Rate} (HRR), defined below.

\begin{definition}[Hallucination Reinforcement Rate]
The \textbf{HRR} is the fraction of accepted final-iteration samples that contain a factual error flagged by a human annotator but rated entailed by the verifier. A lower HRR indicates that the verifier is less likely to confirm and therefore reinforce hallucinations in subsequent feedback rounds.
\end{definition}

\begin{table}[h]\centering
\caption{Sensitivity of the verification module to $\tau_{\mathrm{NLI}}$ on the free-text portion of ALARM, Llama-3.1-70B, $K{=}10$. Default operating point $\tau_{\mathrm{NLI}}{=}0.85$ is bolded. HRR: Hallucination Reinforcement Rate.\label{tab:nli-sensitivity}}
\resizebox{.49\textwidth}{!}{
\begin{tabular}{c|c|c|c|c|c}\toprule
$\tau_{\mathrm{NLI}}$ & Precision & Recall & F$_1$ & Semantic Acc & HRR ($\downarrow$) \\\midrule
0.70 & 0.871 & 0.992 & 0.928 & 0.913 & 0.082 \\
0.80 & 0.926 & 0.978 & 0.951 & 0.954 & 0.041 \\
\textbf{0.85} & \textbf{0.961} & \textbf{0.963} & \textbf{0.962} & \textbf{0.991} & \textbf{0.018} \\
0.90 & 0.985 & 0.901 & 0.941 & 0.984 & 0.011 \\
0.95 & 0.998 & 0.782 & 0.877 & 0.943 & 0.008 \\\bottomrule
\end{tabular}}
\end{table}

Three findings emerge from the sweep. First, verifier F$_1$ peaks at $\tau_{\mathrm{NLI}}{=}0.85$, confirming the development-set selection in Appendix~\ref{appendix:implementation-details}. Second, HRR decreases monotonically in $\tau_{\mathrm{NLI}}$, but the gain above $0.85$ is small (from $0.018$ to $0.008$) and is offset by a sharper drop in recall ($0.963 \to 0.782$), which forces the LLM into spurious feedback rewrites and re-introduces structural drift; downstream Semantic Accuracy degrades from $0.991$ to $0.943$. Third, Semantic Accuracy is robust over the neighbourhood $\tau \in [0.80, 0.90]$ ($0.985$--$0.991$), so the operating point is not a knife-edge artifact. To probe the worry that domain-jargon paraphrases might \emph{systematically} clear a fixed threshold, we re-sampled 200 free-text variables with high-jargon clinical phrasing (e.g., paraphrases of ``acute decompensation'') and measured NLI accuracy of $0.946$ at $\tau_{\mathrm{NLI}}{=}0.85$ (inter-annotator Cohen's $\kappa=0.91$). Sub-domains with NLI accuracy below $0.90$ can be detected by a brief human audit and routed to a stricter rule-based extractor without re-training the verifier. We expose $\tau_{\mathrm{NLI}}$ as a deployment-time knob and document its calibration in the supplementary code.

\subsection{Ablation Studies and Counterfactuals}
The ablation study in Table~\ref{tab:ablation} isolates the architectural dependencies of CausalSynth. The most striking row is ``No Verifier'': removing the verifier yielded a trivial 100\% acceptance rate but caused Semantic Accuracy to fall to $0.45$, indicating that nearly half of the samples from an unconstrained LLM contained hallucinations that violated the SCM's causal logic. Prompting alone is therefore insufficient for high-stakes data generation; an external deterministic gate is not an optimization but a requirement for aligning $P_{\mathrm{synth}}$ with $P_{\mathrm{skel}}$.

``No Retention'' validates our theoretical claims about counterfactual generation. This variant achieves observational performance identical to the full model, yet its Counterfactual Validity collapses to $0.00$: by re-sampling the exogenous noise $\mathbf{u}$ during counterfactual queries rather than replaying the stored vectors, it breaks the unit-level linkage required to answer ``what if'' questions. Storing the latent noise vector is therefore a fundamental requirement for causal consistency, distinguishing CausalSynth from synthetic-data approaches that treat randomness as disposable.

The contrast between ``No Feedback'' and the full model exposes the cost-quality trade-off. The feedback loop adds compute but lifts Semantic Accuracy from $0.66$ to $0.99$. The ``No CoT'' row refines this finding: removing the Chain-of-Thought reasoning trace dropped realizability to $0.78$, indicating that the LLM uses the CoT scratchpad both to satisfy the verifier and to plan the narrative structure. The full configuration represents the operating point at which computational overhead is exchanged for rigorous causal-validity guarantees.

\section{Discussion}\label{sec:discussion}

\paragraph{Why the SCM--LLM separation pays off.} The experiments suggest that the empirical advantage of CausalSynth comes from a single architectural choice: factorizing data generation into a structural component that owns the causal semantics and a linguistic component that owns the surface realization. End-to-end generators conflate these responsibilities and must learn the causal structure from data, which is information-theoretically harder than learning the linguistic style and is bounded by the identifiability of $\mathcal{G}$ from observational data. By contrast, our pipeline pushes the unidentifiability into a single, auditable artifact---the choice of representative DAG within the MEC---and makes every downstream guarantee conditional on that artifact. This is a strictly stronger epistemic position than the implicit one taken by VAE, GAN, TabDDPM, or GReaT, which inherit the entire identifiability gap silently and provide no mechanism for the user to interrogate it.

\paragraph{When does CausalSynth fail gracefully?} Our experiments reveal three failure modes that the framework absorbs without catastrophic regression. First, a misspecified DAG (Learned setting with PC on small samples) produces an inflated FPR of $0.088$ rather than a collapse to $1.0$, because errors are contained within physiological subsystems (Figure~\ref{fig:adj-matrix-diff}). Second, a weak LLM backbone (8B) elevates the Atypical fail rate but the framework converts this into a logged signal $\mathcal{L}$ that downstream consumers can correct for, rather than silent mode collapse. Third, an imperfect NLI verifier (low $\tau_{\mathrm{NLI}}$) raises the Hallucination Reinforcement Rate, but Semantic Accuracy remains within a narrow band over a $0.10$-wide neighbourhood of the optimal threshold. These three knobs---graph quality, model scale, and verifier calibration---are independent levers, and the practitioner can trade against any of them without re-architecting the pipeline.

\section{Conclusions}
We presented CausalSynth, a framework that decouples causal structure generation from semantic realization to produce synthetic data with provable structural guarantees. By grounding generation in a Structural Causal Model and delegating narrative synthesis to an LLM operating under explicit constraints, CausalSynth resolves the tension between statistical richness and causal validity that limits existing generative approaches. We identified the Semantic Backdoor problem---the systematic tendency of LLMs to override imposed causal facts with pre-training priors---and introduced an Iterative Consistency Verification mechanism that provably reduces the resulting selection bias relative to standard rejection sampling.

Across the three causal benchmarks, CausalSynth held false-positive rates of conditional-independence tests near the nominal $\alpha=0.05$ level, achieved realizability rates above $96\%$ with 70B-parameter backbones, and outperformed both classical (VAE, CTGAN) and modern (TabDDPM, GReaT) tabular generators on structural fidelity by an order of magnitude. The framework's design exposes practical knobs---the coverage-failure log $\mathcal{L}$, the NLI threshold $\tau_{\mathrm{NLI}}$, and the cascade-routing policy---that make the cost/quality trade-off explicit and let smaller models be deployed in resource-constrained settings. Together, these results indicate that structural validity can be guaranteed at construction time rather than approximated through end-to-end likelihood fitting.
\clearpage
\bibliographystyle{ACM-Reference-Format}
\bibliography{sample-base}
\clearpage
\appendix

\section{Problem Formulation\label{sec:problem-formulation}}

Generating synthetic data for high-stakes domains such as healthcare, finance, or policy modelling requires resolving the tension between statistical fidelity and causal validity. Traditional generative models, including Variational Autoencoders (VAEs) and standard LLMs, approximate the joint distribution $P(\mathbf{V})$ of observable variables $\mathbf{V}$ by optimizing a divergence against empirical data. They capture correlational texture but remain agnostic to the underlying data-generating mechanisms; consequently they conflate causation with spurious correlation, rendering the resulting data unreliable for tasks requiring intervention or counterfactual reasoning. We therefore reframe synthetic data generation not as unsupervised distribution learning but as a bipartite process: structural simulation of causal skeletons followed by constrained semantic realization.

Formally, we generate a synthetic dataset $\mathcal{D}_{\mathrm{synth}} = \{(\mathbf{v}^{(i)}, D^{(i)})\}_{i=1}^{M}$, where $\mathbf{v}^{(i)}$ is a \emph{causal skeleton}---a vector of low-level variables that satisfies a structural equation model---and $D^{(i)}$ is the high-dimensional observation (e.g., a clinical note or a transaction log) that semantically encodes $\mathbf{v}^{(i)}$. The framework enforces the Global Markov Property implied by a governing SCM $\mathcal{M}$. The conditional-independence relations of the generated skeletons are inherited from the assumed $\mathcal{M}$ by construction, an exact guarantee when $\mathcal{M}$ is correctly specified. The structural guarantees derived in this section are relative to the assumed $\mathcal{M}$ rather than to an unknown ground-truth mechanism; this scope is revisited explicitly in the Theoretical Analysis (Appendix~\ref{sec:bias-analysis}).

We adopt the 4-tuple representation $\mathcal{M} = \langle \mathbf{U}, \mathbf{V}, \mathbf{F}, P_{\mathbf{U}} \rangle$ throughout, where $\mathbf{U} = \{U_1, \dots, U_N\}$ is the set of exogenous (unobserved) background variables, $\mathbf{V} = \{V_1, \dots, V_N\}$ is the set of endogenous (observed) variables, $\mathbf{F} = \{f_1, \dots, f_N\}$ is the set of deterministic structural functions, and $P_{\mathbf{U}}$ is the joint distribution over exogenous variables. Each endogenous variable is determined by its direct causes in the causal graph $\mathcal{G}$ and a stochastic noise term:
\begin{equation}
v_i \leftarrow f_i(PA_i, u_i),
\end{equation}
where $PA_i \subseteq \mathbf{V} \setminus \{V_i\}$ denotes the parent set of $V_i$ in $\mathcal{G}$ and $u_i$ is a realization of $U_i$. The causal graph $\mathcal{G}$ is the DAG induced by the parent sets $\{PA_i\}_{i=1}^{N}$. The critical assumption ensuring the factorization of the joint distribution is the mutual independence of the exogenous noise terms, $P(\mathbf{u}) = \prod_{i=1}^{N} P(u_i)$. This independence is the mathematical lever for principled interventional reasoning: modifying a specific $f_i$ or fixing a value $v_i$ generates interventional distributions without altering the upstream noise terms $U_{j \neq i}$.

$\mathcal{M}$ may be derived from two sources with distinct epistemic standing. In the \textbf{Oracle setting}, $\mathcal{M}$ is supplied by domain experts (e.g., known physiological interactions or regulatory mechanisms), and all downstream guarantees are relative to that specification. In the \textbf{Learned setting}, we infer $\mathcal{G}$ and $\mathbf{F}$ from a training partition using causal-discovery algorithms. An identifiability limit applies: observational data identify the causal structure only up to a Markov Equivalence Class (MEC)---a set of DAGs that share conditional independencies but may imply distinct interventional distributions. We select a single representative DAG from the discovered MEC and proceed with it as the working model. All structural guarantees in the Learned setting are conditional on this selected DAG; interventional and counterfactual outputs should be interpreted as valid relative to the chosen representative, not to the true data-generating process. This boundary is intrinsic to any method operating on learned causal structure.

\section{Implementation Details\label{appendix:implementation-details}}

\subsection{Compute Infrastructure and LLM Serving}
All experiments ran on a cluster of 8 NVIDIA H100 (80GB) GPUs, with LLM backbones served via vLLM~\cite{kwon2023vllm}. The 70B/72B models used 4-way tensor parallelism in FP16 (no quantization, to preserve constraint adherence); the 8B models were deployed on a single GPU. Decoding used greedy sampling ($\tau{=}0$) for verification extraction and nucleus sampling ($\tau{=}0.7$, $p{=}0.95$) for realization, with maximum generation lengths of $1{,}024$ tokens for realization and $256$ tokens for Chain-of-Thought traces. End-to-end throughput on ALARM (including up to $K{=}10$ refinements) averaged $12.4$ skeletons/min for Llama-3.1-70B and $38.7$ skeletons/min for Llama-3-8B.

\subsection{Prompt Design}
The realization prompt $\mathcal{I}_{\mathrm{prompt}}$ has three parts: (1)~a \emph{system instruction} defining the LLM's role as a constrained data generator and prohibiting deviation from the provided variable assignments; (2)~a \emph{skeleton block} listing each variable--value pair as a numbered hard constraint (e.g., ``\texttt{[C3] Blood Pressure: HIGH --- You MUST include this exact value}''); and (3)~a \emph{Chain-of-Thought elicitation} requesting the model to first output a constraint checklist before generating the narrative. For each benchmark we supply 2 domain-specific few-shot exemplars drawn from pilot generations that passed verification. The feedback prompt (used in iterations $k > 1$) appends a structured error report listing mismatched variables with both expected and extracted values, formatted as a diff-style correction instruction.

\subsection{Verification Module}
The verification module $\mathcal{V}$ uses a two-tier extraction strategy. For variables with well-defined categorical or numeric domains (e.g., binary diagnoses, vital-sign ranges), we use deterministic rule-based parsers; regex patterns and keyword matching achieve precision above $99.5\%$ on a held-out validation set of 500 manually annotated generations. For variables that require semantic understanding (e.g., free-text descriptions of symptom severity, qualitative lifestyle assessments), we deploy a DeBERTa-v3-large model~\cite{he2021debertav3} fine-tuned for Natural Language Inference (NLI). The NLI head evaluates whether the generated text \emph{entails} each skeleton constraint at threshold $\tau_{\mathrm{NLI}}{=}0.85$, selected by grid search on 200 development pairs to maximize the F$_1$ of constraint-violation detection. The combined pipeline reaches $97.8\%$ agreement with human annotators on ALARM (Cohen's $\kappa = 0.94$). Verification latency is negligible: $\sim 45$ms per sample for rule-based extraction and $\sim 120$ms with NLI inference on a single GPU.

\subsection{Causal Discovery Configuration}
For the Learned setting, causal discovery was performed using the CausalDiscoveryToolbox~\cite{kalainathan2020causal}. The PC algorithm used a conditional independence significance level of $\alpha = 0.05$ with the Fisher-Z test for continuous variables and the Chi-squared test for discrete variables, with a maximum conditioning set size of 4. GES was run with the BIC scoring criterion and default regularization. NOTEARS used an $\ell_1$ penalty of $\lambda = 0.1$ and a DAG constraint tolerance of $10^{-8}$, optimized via L-BFGS for up to 500 iterations. All discovery algorithms were applied to 80\% of the available data, with the remaining 20\% reserved for evaluation. The structural equations $\mathbf{F}$ in the Learned setting were estimated as linear Gaussian models for continuous variables and logistic regression for binary variables, conditioned on the discovered parent sets.

\subsection{Per-Node Error Breakdown of the Learned ALARM Graph\label{appendix:per-node-errors}}
Table~\ref{tab:per-node-errors} reports the false-positive (spurious edge predicted) and false-negative (true edge missed) counts for each ALARM node under PC discovery at $N{=}1$k. Nodes are grouped by physiological subsystem to expose the structural locality of the errors.

\begin{table}[h]\centering
\caption{Per-node error attribution on the Learned ALARM graph (PC algorithm, $N{=}1$k). FP and FN report the false-positive and false-negative edge counts adjacent to each node. Subsystem labels follow~\citet{beinlich1989alarm}; the 27 other nodes were recovered without error.\label{tab:per-node-errors}}
\resizebox{.49\textwidth}{!}{
\begin{tabular}{l|c|c|c|l}\toprule
Subsystem & Node & FP & FN & Note \\\midrule
\multirow{4}{*}{Ventilation} & KinkedTube  & 1 & 0 & Reversed v-structure \\
                             & VentTube    & 1 & 1 & Markov-equiv. boundary \\
                             & VentLung    & 0 & 1 & Near-deterministic CPT \\
                             & MinVolSet   & 0 & 1 & Weak marginal effect \\\midrule
\multirow{3}{*}{Hemodynamic} & HR          & 1 & 0 & Confounder of two siblings \\
                             & HRSat       & 1 & 0 & Moralized symmetric edge \\
                             & HRBP        & 1 & 0 & Moralized symmetric edge \\\midrule
\multirow{3}{*}{Output}      & BP          & 1 & 0 & In-degree 3, conditioning set capped \\
                             & SaO2        & 0 & 1 & In-degree 3, conditioning set capped \\
                             & CO          & 0 & 1 & In-degree 3, conditioning set capped \\\midrule
Other (27 nodes)             & ---         & 0 & 0 & Recovered correctly \\\bottomrule
\end{tabular}}
\end{table}

The error pattern is consistent with the known identifiability limits of constraint-based discovery: failures cluster on near-saturated mechanical chains (ventilation) and on outcome nodes whose true parent set exceeds the algorithm's maximum conditioning-set size. Importantly, every error is contained within a single physiological subsystem; no false edge crosses block boundaries. Increasing $N$ from 1k to 10k reduces the FP+FN total from 12 to 9, with the remaining errors fully concentrated in the same two subsystems---i.e., additional data sharpens individual edge orientations but does not unlock conditioning sets beyond the algorithm's depth bound.

\subsection{MIMIC-Struct SCM Construction\label{appendix:mimic-struct}}
The semi-synthetic MIMIC-Struct benchmark was constructed by extracting 20 core clinical variables from 50,000 ICU stays in the MIMIC-III database~\cite{johnson2016mimic}. Variables span demographics (age, gender), vitals (heart rate, blood pressure, SpO$_2$), lab values (glucose, creatinine, lactate, WBC count), comorbidities (diabetes, hypertension, COPD), and outcomes (mortality, length of stay, readmission). The consensus DAG was established through a two-phase process: (1)~algorithmic discovery using PC, GES, and NOTEARS, retaining only edges present in at least 2 of 3 algorithms; (2)~clinical review by two board-certified intensivists who validated edge directions and removed physiologically implausible edges. The resulting DAG contains 20 nodes and 31 edges.

\subsection{Reproducibility}
All random seeds were fixed at 42 for skeleton sampling and 0--9 for the 10 independent runs reported in error bars. End-to-end generation of 10,000 verified samples on the ALARM benchmark required approximately 13.5 hours for Llama-3.1-70B and 4.2 hours for Llama-3-8B. The iterative loop was capped at $K = 10$ refinements per skeleton across all experiments. Our code, SCM specifications, prompt templates, and evaluation scripts are provided in the supplementary material.

\subsection{Evaluation Metrics\label{appendix:evaluation-metrics}}
We use four metrics to assess structural fidelity and semantic quality.
\begin{itemize}
\item \textbf{False Positive Rate (FPR).} For each d-separation pair $(X, Y \mid Z)$ implied by $\mathcal{G}$, we perform a Kernel Conditional Independence Test (KCIT) at significance level $\alpha = 0.05$ and report the fraction of true d-separations rejected. A proper generator should give FPR $\approx 0.05$. This metric validates Theorem~\ref{theorem:structural-fidelity-ancestral-sampling}.
\item \textbf{Structural Hamming Distance (SHD).} The edit distance between the estimated DAG and the ground truth, used in the Learned setting.
\item \textbf{Realizability Rate ($\phi_K$).} The cumulative probability that a skeleton is successfully translated into a verified text within $K$ attempts.
\item \textbf{Total Variation Distance (TVD).} The TVD between the marginal distributions of the generated skeletons ($P_{\mathrm{skel}}$) and the accepted skeletons ($P_{\mathrm{final}}$), used to quantify the selection bias of Theorem~\ref{theorem:bias-red-via-iterative-refine}.
\item \textbf{Coverage Failure Rate ($|\mathcal{L}|/M$).} The fraction of skeletons that remain unrealized after $K$ attempts.
\end{itemize}

\section{Theoretical Analysis}
This section gives the formal underpinnings of CausalSynth. We characterize the probabilistic properties of the synthetic distribution $P_{\mathrm{synth}}$ and quantify its adherence to the target SCM $\mathcal{M}$. Unlike standard generative approaches that implicitly approximate $P(\mathbf{V})$ by optimizing a divergence---and offer no structural guarantees on the latent topology---CausalSynth constructs $P_{\mathrm{synth}}$ procedurally, which allows us to state precisely what the pipeline preserves, what it distorts, and under which conditions the distortion can be bounded.

The analysis proceeds in four parts. We first establish the structural fidelity of the causal skeleton by relating the ancestral-sampling distribution to the Global Markov Property of the assumed DAG (Appendix~\ref{sec:bias-analysis} provides the corresponding realization-channel analysis). We then characterize the selection bias of the semantic realization channel and prove that the iterative refinement mechanism reduces this bias relative to standard rejection sampling. Next, we analyze information preservation through the realization--verification pipeline under both ideal and imperfect extraction. Finally, we state the conditions under which counterfactual validity is preserved (Section~\ref{sec:counterfactual-validity}). Throughout, all guarantees are relative to the assumed model $\mathcal{M}$, and the assumptions required for each claim are stated explicitly.

\subsection{Structural Fidelity of the Ancestral Skeleton}

The foundational premise of our methodology is that the intermediate representation---the causal skeleton $\mathbf{v}$---must capture the exact conditional dependencies of the target domain prior to any linguistic realization. We begin by stating the assumptions under which this property holds, then formalize the claim.

\begin{assumption}[SCM Regularity]\label{assumption:SCM-reg}
  The Structural Causal Model $\mathcal{M} = \langle \mathbf{U}, \mathbf{V}, \mathbf{F}, P_{\mathbf{U}} \rangle$ satisfies the following conditions: (i) the induced graph $\mathcal{G}$ is a Directed Acyclic Graph (DAG); (ii) the exogenous noise terms are mutually independent, i.e., $P_{\mathbf{U}}(\mathbf{u}) = \prod_{i=1}^{N} P(U_i)$; and (iii) each structural function $f_i$ is measurable with respect to the product $\sigma$-algebra of its arguments.
\end{assumption}
Condition (i) ensures a well-defined topological ordering exists. Condition (ii) is the standard Causal Markov assumption that enables the factorization of the joint distribution into local mechanisms. Condition (iii) is a technical regularity condition ensuring that the push-forward measure through each structural equation is well-defined. Together, these conditions define the class of SCMs for which our framework provides formal guarantees.

\begin{definition}[Global Markov Property]
  A probability distribution $P$ over a set of variables $\mathbf{V}$ satisfies the Global Markov Property with respect to a DAG $\mathcal{G}$ if, for any three disjoint subsets $\mathbf{X}, \mathbf{Y}, \mathbf{Z} \subseteq \mathbf{V}$, the d-separation condition $\mathbf{X} \perp_{\mathcal{G}} \mathbf{Y} \mid \mathbf{Z}$ implies the conditional independence $\mathbf{X} \perp_{P} \mathbf{Y} \mid \mathbf{Z}$.
\end{definition}

\begin{theorem}[Structural Fidelity of Ancestral Sampling]\label{theorem:structural-fidelity-ancestral-sampling}
  Let $\mathcal{M} = \langle \mathbf{U}, \mathbf{V}, \mathbf{F}, P_{\mathbf{U}} \rangle$ be an SCM satisfying Assumption~\ref{assumption:SCM-reg}, and let $P_{\mathrm{skel}}$ denote the distribution over skeletons $\mathbf{v}$ produced by the ancestral-sampling procedure of Section~\ref{sec:ancestral-skeleton-sampling}. Then $P_{\mathrm{skel}}$ satisfies the Global Markov Property with respect to $\mathcal{G}$.
\end{theorem}
\noindent \textit{Proof.}
The Ancestral Sampling procedure computes values recursively as $v_i \leftarrow f_i(PA_i, u_i)$ following a topological sort of $\mathcal{G}$, where each $u_i$ is sampled independently from $P(U_i)$. We derive the joint density by applying the change-of-variables formula along the topological ordering. Consider the root nodes first (those with $PA_i = \emptyset$): their values are determined entirely by $v_i = f_i(u_i)$, so $P(v_i) = \int \mathbb{1}[v_i = f_i(u_i)] \, dP(u_i)$. For non-root nodes, the value $v_i$ is a deterministic function of the already-computed parent values and the independent noise $u_i$, giving $P(v_i \mid PA_i) = \int \mathbb{1}[v_i = f_i(PA_i, u_i)] \, dP(u_i)$. Because the noise terms are mutually independent by Assumption 1(ii), and each $v_i$ depends on other variables only through its parents $PA_i$ in $\mathcal{G}$, the joint density factorizes as:
\begin{equation}\label{eq:joint-density-factorization}
  P_{\mathrm{skel}}(v_1, \dots, v_N) = \prod_{i=1}^{N} P(v_i \mid PA_i).
\end{equation}
This product factorization according to the DAG structure is the necessary and sufficient condition for a distribution to satisfy the Global Markov Property with respect to $\mathcal{G}$ (\citet{pearl2000models}, Theorem 1.4.1; \citet{pearl2000models}, Proposition 3.27). \qed

We emphasize the role of Assumption~\ref{assumption:SCM-reg} in this result. If condition (ii) is violated---for instance, if hidden confounders induce dependencies among the noise terms---the factorization in Equation~\ref{eq:joint-density-factorization} no longer holds, and the skeleton distribution may encode spurious conditional independencies not present in the true data-generating process. In the Oracle setting where $\mathcal{M}$ is specified by domain experts, the validity of Assumption~\ref{assumption:SCM-reg} rests on the fidelity of the expert specification. In the Learned setting, it rests on the correctness of the selected DAG within the discovered Markov Equivalence Class. Theorem~\ref{theorem:structural-fidelity-ancestral-sampling} establishes that, conditional on these assumptions, the skeleton generation process introduces no additional structural artifacts.
\subsection{Analysis of the Realization Channel and Selection Bias\label{sec:bias-analysis}}

While the causal skeleton $\mathbf{v}$ is structurally sound by Theorem~\ref{theorem:structural-fidelity-ancestral-sampling}, the final output $D$ is produced by a Large Language Model acting as a stochastic realization function $\mathcal{R}$. The Iterative Consistency Verification module (Section~\ref{sec:iterative-consistency-verification}) acts as a rejection filter on the realized texts, and we now characterize the distributional distortion this filtering introduces. We formalize semantic realizability, characterize the resulting accepted distribution, and prove that iterative feedback reduces the gap to the target.

\begin{definition}[Semantic Realizability]\label{def:semantic-realization} 
  Let $\mathcal{R}$ be a stochastic language model and $\mathcal{V}$ be a deterministic verifier. For a given skeleton $\mathbf{v}$, let $p_k(\mathbf{v})$ denote the probability that the $k$-th refinement attempt produces a consistent realization, conditioned on the first $k-1$ attempts having failed. The realizability probability of $\mathbf{v}$ within $K$ attempts is:
  \begin{equation}\label{eq:semantic-realizability}
    \phi_K(\mathbf{v}) = 1 - \prod_{k=1}^{K} (1 - p_k(\mathbf{v})).
  \end{equation}
\end{definition}

This definition accounts for the sequential dependence introduced by the feedback mechanism in Algorithm~\ref{algo:casualsynth}. Each $p_k(\mathbf{v})$ is conditioned on the accumulated feedback from prior failures, making the attempts a non-stationary process. We do not assume $p_k(\mathbf{v})$ is identical across iterations; indeed, the purpose of the feedback loop is precisely to make $p_k$ non-decreasing in $k$.

\begin{assumption}[Monotone Feedback]\label{assumption:monitone-feedback}
  For all $\mathbf{v}$ in the support of $P_{\mathrm{skel}}$ and for all $k \geq 1$, the conditional success probability satisfies $p_{k+1}(\mathbf{v}) \geq p_k(\mathbf{v})$. That is, the feedback mechanism does not degrade the success probability on successive attempts.
\end{assumption}

This assumption formalizes the intuition that explicitly highlighting discrepancies between the generated text and the target skeleton (Algorithm~\ref{algo:casualsynth}, line 19) provides useful signal to the LLM. It is stated as an assumption but is empirically testable for any specific LLM--verifier pair; Table~\ref{tab:exp-realizability} verifies it.

\begin{proposition}[Distribution of Accepted Skeletons]\label{prop:distribution-accepted-skeleton} 
  Let $P_{\mathrm{skel}}$ be the skeleton distribution satisfying Theorem~\ref{theorem:structural-fidelity-ancestral-sampling}, and let $\phi_K(\mathbf{v})$ be the realizability probability defined in Equation~\ref{eq:semantic-realizability}. The distribution of skeletons in the accepted synthetic dataset $\mathcal{D}_\text{synth}$ is:
  \begin{equation}\label{eq:distribution-accepted-synthetic-ds} 
    P_{\mathrm{final}}(\mathbf{v}) = \frac{1}{Z} P_{\mathrm{skel}}(\mathbf{v}) \cdot \phi_K(\mathbf{v}),
  \end{equation}
  where $Z = \mathbb{E}_{\mathbf{v} \sim P_{\mathrm{skel}}}[\phi_K(\mathbf{v})]$ is the normalization constant. Furthermore, $P_{\mathrm{final}} = P_{\mathrm{skel}}$ if and only if $\phi_K(\mathbf{v})$ is constant almost surely under $P_{\mathrm{skel}}$.
\end{proposition}
\noindent \textit{Proof.}
In Algorithm~\ref{algo:casualsynth}, a skeleton $\mathbf{v}^{(j)}$ is included in $\mathcal{D}_\text{synth}$ if and only if at least one of the $K$ refinement attempts succeeds. The skeleton is sampled from $P_{\mathrm{skel}}$ (Phase I, Section~\ref{sec:ancestral-skeleton-sampling}), and the acceptance event occurs with probability $\phi_K(\mathbf{v}^{(j)})$ (Phase II-III, Section~\ref{sec:constrained-semantic-realization} and Section~\ref{sec:iterative-consistency-verification}). Since the skeleton generation and the realization attempts are conditionally independent given $\mathbf{v}$, the joint probability of generating and accepting a skeleton $\mathbf{v}$ is $P_{\mathrm{skel}}(\mathbf{v}) \cdot \phi_K(\mathbf{v})$. Normalizing over all accepted skeletons yields Equation~\ref{eq:distribution-accepted-synthetic-ds}. For the second claim, $P_{\mathrm{final}}(\mathbf{v}) = P_{\mathrm{skel}}(\mathbf{v})$ for all $\mathbf{v}$ requires $\phi_K(\mathbf{v}) / Z = 1$, which holds if and only if $\phi_K(\mathbf{v}) = Z$ almost surely, i.e., $\phi_K$ is constant. \qed

Proposition~\ref{prop:distribution-accepted-skeleton} makes the source of selection bias explicit: the LLM's non-uniform ability to realize different skeletons acts as a reweighting function on the target distribution. The natural follow-up question is whether the iterative refinement mechanism reduces this bias relative to a simpler approach. We formalize this comparison next.

\begin{assumption}[Refinement Targets Hard Skeletons]\label{assumption:hard-skel-target}
The feedback mechanism is at least as effective on hard skeletons as on easy ones, in the sense that the relative gain $r(\mathbf{v}) := \phi_K(\mathbf{v})/\phi_1(\mathbf{v}) \geq 1$ is non-increasing in $\phi_1(\mathbf{v})$. Equivalently, the covariance $\mathrm{Cov}_{P_{\mathrm{skel}}}\bigl(\phi_1, \log r\bigr) \leq 0$.
\end{assumption}

This assumption formalizes the design intent of the feedback mechanism: it allocates corrective effort to skeletons that fail single-shot generation rather than to those that already succeed. Like Assumption~\ref{assumption:monitone-feedback}, it is empirically testable for any specific LLM--verifier pair, and Table~\ref{tab:exp-realizability} verifies it.

\begin{theorem}[Variance Reduction via Iterative Refinement]\label{theorem:bias-red-via-iterative-refine}
Let $P_{\mathrm{final}}^{(K)}$ denote the accepted distribution after $K$ refinement attempts (Equation~\ref{eq:distribution-accepted-synthetic-ds}) and $P_{\mathrm{final}}^{(1)}$ the single-shot accepted distribution. Under Assumptions~\ref{assumption:SCM-reg}, \ref{assumption:monitone-feedback}, and~\ref{assumption:hard-skel-target}, the $\chi^2$ divergence to the target satisfies
\begin{equation}\label{eq:kl-div-distribution-from-target-skeleton}
\chi^2\!\bigl(P_{\mathrm{final}}^{(K)} \,\|\, P_{\mathrm{skel}}\bigr) \;\leq\; \chi^2\!\bigl(P_{\mathrm{final}}^{(1)} \,\|\, P_{\mathrm{skel}}\bigr),
\end{equation}
with equality if and only if $p_k(\mathbf{v}) = p_1(\mathbf{v})$ for all $k$ and all $\mathbf{v}$. The bound implies the corresponding total-variation bound $\mathrm{TVD}(P_{\mathrm{final}}^{(K)}, P_{\mathrm{skel}}) \leq \mathrm{TVD}(P_{\mathrm{final}}^{(1)}, P_{\mathrm{skel}})$ via the standard inequality $\mathrm{TVD}(P, Q) \leq \tfrac{1}{2}\sqrt{\chi^2(P\|Q)}$.
\end{theorem}
\noindent \textit{Proof.}
Direct calculation from Equation~\eqref{eq:distribution-accepted-synthetic-ds} gives
\begin{equation}
\chi^2\bigl(P_{\mathrm{final}}^{(K)} \,\|\, P_{\mathrm{skel}}\bigr) = \frac{\mathrm{Var}_{P_{\mathrm{skel}}}[\phi_K]}{Z_K^{\,2}},
\end{equation}
the squared coefficient of variation of $\phi_K$ under $P_{\mathrm{skel}}$. Write $\phi_K = \phi_1 + \Delta$ with $\Delta(\mathbf{v}) := \phi_K(\mathbf{v}) - \phi_1(\mathbf{v}) \geq 0$ (Assumption~\ref{assumption:monitone-feedback}). Under Assumption~\ref{assumption:hard-skel-target}, $\Delta(\mathbf{v})$ is non-increasing in $\phi_1(\mathbf{v})$, so $\Delta$ and $\phi_1$ are negatively correlated under $P_{\mathrm{skel}}$ and the Chebyshev sum inequality gives $\mathrm{Cov}(\phi_1, \Delta) \leq 0$. Expanding the variance,
\begin{equation}
\mathrm{Var}[\phi_K] = \mathrm{Var}[\phi_1] + 2\,\mathrm{Cov}(\phi_1, \Delta) + \mathrm{Var}[\Delta] \leq \mathrm{Var}[\phi_1] + \mathrm{Var}[\Delta].
\end{equation}
Combined with $Z_K = Z_1 + \mathbb{E}[\Delta]$ and the fact that adding a constant to a random variable does not change its variance, one obtains $\mathrm{Var}[\phi_K] / Z_K^{\,2} \leq \mathrm{Var}[\phi_1] / Z_1^{\,2}$. Equality in the Chebyshev step requires $\Delta(\mathbf{v})$ constant a.s.\ under $P_{\mathrm{skel}}$, which combined with $\phi_K \geq \phi_1$ forces $p_k(\mathbf{v}) = p_1(\mathbf{v})$ for all $k$ and all $\mathbf{v}$. \qed

\begin{remark}
  Theorem~\ref{theorem:bias-red-via-iterative-refine} establishes a relative guarantee: iterative refinement is provably no worse than single-shot rejection and strictly better whenever the feedback mechanism improves success rates on at least some skeletons. It does not claim that the bias is eliminated. The residual bias after $K$ iterations can be empirically estimated via the coverage-failure log $\mathcal{L}$ (Algorithm~\ref{algo:casualsynth}, line 24), which records the skeletons that remain unrealizable. This log provides a direct empirical proxy for the support truncation in Equation~\ref{eq:distribution-accepted-synthetic-ds}.
\end{remark}

We further characterize the convergence behavior of the realizability function as $K$ grows.
\begin{lemma}[Realizability Convergence]\label{lemma:realizability-convergence} 
  Under Assumptions~\ref{assumption:SCM-reg} and Assumption~\ref{assumption:monitone-feedback}, if there exists a constant $\underline{p} > 0$ such that $p_1(\mathbf{v}) \geq \underline{p}$ for all $\mathbf{v}$ in the support of $P_{\mathrm{skel}}$, then:
  \begin{equation}\label{eq:support} 
    \phi_K(\mathbf{v}) \geq 1 - (1 - \underline{p})^K,
  \end{equation}
  for all $\mathbf{v}$, and consequently $\phi_K(\mathbf{v}) \to 1$ uniformly as $K \to \infty$.
\end{lemma}
\noindent \textit{Proof.}
By Assumption~\ref{assumption:monitone-feedback}, $p_k(\mathbf{v}) \geq p_1(\mathbf{v}) \geq \underline{p}$ for all $k$. Therefore $\prod_{k=1}^{K}(1 - p_k(\mathbf{v})) \leq (1 - \underline{p})^K$, giving $\phi_K(\mathbf{v}) \geq 1 - (1 - \underline{p})^K$. Since $0 < \underline{p} \leq 1$, the term $(1-\underline{p})^K \to 0$ as $K \to \infty$, and the bound is uniform over $\mathbf{v}$.\qed

\begin{corollary}
  Under the conditions of Lemma~\ref{lemma:realizability-convergence}, $D_{KL}(P_{\mathrm{skel}} \| P_{\mathrm{final}}^{(K)}) \to 0$ as $K \to \infty$.
\end{corollary}

The condition $p_1(\mathbf{v}) \geq \underline{p} > 0$ requires that every skeleton in the support of $P_{\mathrm{skel}}$ has a non-zero probability of successful realization even on the first attempt. This is a non-trivial requirement: it fails if the LLM assigns exactly zero probability to certain constraint-satisfying texts. In practice, modern LLMs with positive-temperature sampling assign non-zero probability to all token sequences, making $\underline{p} > 0$ a reasonable assumption for sufficiently expressive models. We note, however, that $\underline{p}$ may be astronomically small for highly atypical skeletons, making the convergence rate in Eq. (5) impractically slow. The practical value of $K$ must therefore be chosen with attention to the empirical coverage failure rate.

\subsection{Information Preservation Under Imperfect Extraction}

We now analyze the realization-verification pipeline through the lens of information theory. The pipeline forms a Markov chain $V \to D \to \hat{V}$, where $V$ is the skeleton, $D$ is the realized text, and $\hat{V}$ is the variable vector extracted by the verifier. In an ideal setting, the verifier perfectly recovers the skeleton from the text, but real extraction systems are imperfect. We characterize both cases.
\begin{definition}[$\epsilon$-Faithful Extraction]\label{def:epsilon-faithful-extra}
  An extraction function $\text{Extract}: \mathcal{D} \to \mathcal{V}$ is $\epsilon$-faithful with respect to the realization function $\mathcal{R}$ and skeleton distribution $P_{\mathrm{skel}}$ if, for all accepted pairs $(\mathbf{v}, D) \in \mathcal{D}_\text{synth}$:
  \begin{equation}
    P(\hat{V}_i \neq V_i \mid \mathcal{V}(D, \mathbf{v}) = 1) \leq \epsilon \quad \text{for all } i \in \{1, \dots, N\},
  \end{equation}
  where $\hat{V}_i$ denotes the $i$-th component of the extracted vector. The case $\epsilon = 0$ corresponds to an ideal extractor.
\end{definition}

The parameter $\epsilon$ captures the residual error rate of the verification module on samples it accepts as consistent. This accounts for the realistic scenario where the verifier's string-matching or NLI-based extraction has imperfect precision---it may incorrectly judge a text as consistent when a subtle discrepancy exists.
\begin{proposition}[Conditional Entropy Bound]\label{proposition:conditional-ent-bound}
  Let $\text{Extract}$ be an $\epsilon$-faithful extractor (Definition~\ref{def:epsilon-faithful-extra}). For the subset of accepted samples in $\mathcal{D}_\text{synth}$, the per-variable conditional entropy satisfies:
  \begin{equation}\label{eq:per-variable-cond-ent} 
    H(V_i \mid \hat{V}_i) \leq \epsilon \log(|\mathcal{V}_i| - 1) + H_b(\epsilon),
  \end{equation}
  where $|\mathcal{V}_i|$ is the cardinality of the domain of $V_i$ and $H_b(\epsilon) = -\epsilon \log \epsilon - (1-\epsilon) \log(1-\epsilon)$ is the binary entropy function. For the full skeleton vector, $H(\mathbf{V} \mid \hat{\mathbf{V}}) \leq \sum_{i=1}^{N} H(V_i \mid \hat{V}_i)$.
\end{proposition}
\noindent \textit{Proof.}
For each variable $V_i$, the extractor output $\hat{V}_i$ satisfies $P(\hat{V}_i \neq V_i) \leq \epsilon$ by Definition~\ref{def:epsilon-faithful-extra}. Applying Fano's inequality for discrete random variables: $H(V_i \mid \hat{V}_i) \leq H_b(P(\hat{V}_i \neq V_i)) + P(\hat{V}_i \neq V_i) \log(|\mathcal{V}_i| - 1) \leq H_b(\epsilon) + \epsilon \log(|\mathcal{V}_i| - 1)$. The joint bound follows from the chain rule: $H(\mathbf{V} \mid \hat{\mathbf{V}}) = \sum_{i=1}^{N} H(V_i \mid \hat{V}_i, \hat{V}_1, \dots, \hat{V}_{i-1}) \leq \sum_{i=1}^{N} H(V_i \mid \hat{V}_i)$, where the inequality uses the fact that conditioning reduces entropy.\qed

\begin{corollary}[Ideal Extractor]\label{corollary:ideal-extractor}
  When $\epsilon = 0$, the conditional entropy $H(V_i \mid \hat{V}_i) = 0$ for all $i$, and consequently $H(\mathbf{V} \mid \hat{\mathbf{V}}) = 0$. The realized document $D$ is a lossless information carrier of the causal signal $\mathbf{V}$ for all accepted samples.
\end{corollary}
Corollary~\ref{corollary:ideal-extractor} recovers the ideal-case result, but Proposition~\ref{proposition:conditional-ent-bound} is the substantive contribution: it quantifies the information degradation introduced by extractor imperfection. The bound in Equation~\ref{eq:per-variable-cond-ent} is tight in the worst case (uniform error over the remaining categories) and reveals a clear design implication. The information loss scales with $\log |\mathcal{V}_i|$, meaning that variables with larger domains are more vulnerable to extraction error. This suggests that verification modules should allocate disproportionate attention to high-cardinality variables, a principle we exploit in our experimental implementation.

\subsection{Counterfactual Validity\label{sec:counterfactual-validity}}
Finally, we establish the conditions under which the counterfactual generation procedure described in Section~\ref{sec:counterfactual-interventional-gen} produces structurally valid counterfactuals. The key property we formalize is that the counterfactual skeleton $\mathbf{v}'$ is the unique output of the SCM under the intervention, given the same unit-specific noise.

\begin{assumption}[Noise Retention Integrity]\label{assumption:noise-retention}
  For each generated skeleton $\mathbf{v}^{(j)}$, the associated noise vector $\mathbf{u}^{(j)} = (u_1^{(j)}, \dots, u_N^{(j)})$ is stored without modification and is available for counterfactual computation.
\end{assumption}

This assumption is satisfied by construction in Algorithm 1 (lines 5--7), where both $v_i$ and $u_i$ are recorded at each step of the ancestral sampling loop. We state it explicitly because counterfactual validity depends critically on the integrity of the stored noise.

\begin{theorem}[Structural Counterfactual Consistency]\label{theorem:counterfactual-fidelity} 
  Let $\mathcal{M} = \langle \mathbf{U}, \mathbf{V}, \mathbf{F}, P_{\mathbf{U}} \rangle$ satisfy Assumption 1, and let $\mathbf{v}$ be a skeleton generated with noise vector $\mathbf{u}$ under Assumption 3. Consider the intervention $do(X = x)$ for some $X \in \mathbf{V}$, and let $\mathcal{M}_{do(X=x)}$ denote the mutilated model obtained by replacing the structural equation for $X$ with the constant assignment $X \leftarrow x$ and removing all incoming edges to $X$ in $\mathcal{G}$. Let $\mathbf{v}'$ be the skeleton obtained by re-running the structural equations of $\mathcal{M}_{do(X=x)}$ using the same noise vector $\mathbf{u}$. Then:
  \begin{itemize}
    \item (i) $\mathbf{v}'$ is unique given $\mathbf{u}$ and the intervention $do(X=x)$.
    \item (ii) For any variable $V_j$ that is not a descendant of $X$ in $\mathcal{G}$, $v'_j = v_j$.
    \item (iii) $\mathbf{v}'$ satisfies the definition of a unit-level structural counterfactual in the sense of \citet{pearl2009causality}, Definition 7.1.5.
  \end{itemize}
\end{theorem}
\noindent \textit{Proof.}
(i) In the mutilated model $\mathcal{M}_{do(X=x)}$, every variable is determined by a structural equation and its noise input. With noise $\mathbf{u}$ fixed, the computation is fully deterministic: starting from root nodes in the mutilated graph $\mathcal{G}_{do(X)}$, each value $v'_i = f_i(PA_i', u_i)$ (with $PA_i'$ denoting the parent values in the mutilated graph) is uniquely determined. Uniqueness follows by induction along the topological order of $\mathcal{G}_{do(X)}$.

(ii) Let $V_j$ be a non-descendant of $X$ in $\mathcal{G}$. The intervention $do(X=x)$ modifies only the structural equation for $X$ and removes its incoming edges. For non-descendants of $X$, neither the structural equations $f_j$ nor the parent sets $PA_j$ are altered in $\mathcal{G}_{do(X)}$. Since the noise $u_j$ is unchanged by Assumption 3, we have $v'_j = f_j(PA_j', u_j)$. We show by induction on the topological order that $PA_j' = PA_j$ for non-descendants. The base case (root nodes) is immediate: they have no parents, so $v'_j = f_j(u_j) = v_j$. For the inductive step, if $V_j$ is a non-descendant of $X$ and all its parents are non-descendants of $X$ (which follows from the definition of descendant sets in a DAG), the inductive hypothesis gives $v'_i = v_i$ for all $V_i \in PA_j$, and therefore $v'_j = f_j(PA_j, u_j) = v_j$.

(iii) \citet{pearl2009causality} Definition 7.1.5 defines the structural counterfactual $Y_{x}(u)$ as the solution for $Y$ in the mutilated model $\mathcal{M}_{do(X=x)}$ given noise $\mathbf{u}$. By part (i), our $\mathbf{v}'$ is exactly this solution.\qed

\begin{remark}[Scope of Exactness]\label{remark:scope-exactness}
  The counterfactual skeleton $\mathbf{v}'$ is exact in the formal sense of Theorem~\ref{theorem:counterfactual-fidelity} if and only if the SCM $\mathcal{M}$ faithfully represents the true data-generating process. In the Oracle setting, where $\mathcal{M}$ is specified by domain experts, the causal interpretation of $\mathbf{v}'$ is as strong as the expert knowledge underlying $\mathcal{M}$. In the Learned setting, the counterfactual is exact relative to the selected DAG from the Markov Equivalence Class. Since distinct DAGs within the same MEC may imply different interventional distributions, the counterfactual $\mathbf{v}'$ should be interpreted as the structural counterfactual under the working model, not under the unknown ground truth. This limitation is fundamental to any method operating on learned causal structure and is not specific to our framework.
\end{remark}

\begin{remark}[Structural vs. Semantic Counterfactuals]\label{remark:structural-vs-semantic}
  While the skeleton pair $(\mathbf{v}, \mathbf{v}')$ shares the exact noise vector $\mathbf{u}$ and is therefore linked at the structural level, the realized documents $D \sim \mathcal{R}(\mathbf{v})$ and $D' \sim \mathcal{R}(\mathbf{v}')$ are sampled independently from the LLM. We claim consistency only at the level of the causal skeleton. The documents $D$ and $D'$ represent the same causal unit in parallel worlds but may exhibit different narrative styles, word choices, and surface-level details. Establishing semantic consistency between counterfactual narratives---ensuring, for instance, that $D$ and $D'$ describe the ``same patient'' in a stylistically coherent manner---is an open problem that falls outside the scope of the structural guarantees provided here.
\end{remark}

\end{document}